\documentclass[twocolumn]{article}

 \newlength\titlebox 
\renewenvironment{abstract}{\centerline{\bf
Abstract}\vspace{0.5ex}\begin{quote}\small}{\par\end{quote}\vskip 1ex}
\usepackage{enumitem}
\setlist{leftmargin=5mm}
\usepackage{enumitem}
\usepackage{times}
\usepackage{helvet}
\usepackage{courier}
\usepackage{soul}\setuldepth{x}
\usepackage{url}
\usepackage[usenames,dvipsnames]{xcolor}

\usepackage[pdftex,
 hyperindex=true,
 colorlinks=true,
 linkcolor={red},
 citecolor={blue},
 pdftitle={Supervisory Control for Behavior Composition with Constraints},
 pdfauthor={Paolo Felli, Nitin Yadav, and Sebastian Sardina},
 pdfkeywords={composition, DES, synthesis},
 ]
{hyperref}

\usepackage{latexsym}
\usepackage{pgf,tikz}
\usepackage{amsmath}
\usepackage{amssymb}
\usepackage{MnSymbol}
\usepackage{xspace}
\usepackage{pifont}

\newcommand{\concat}{\cdot}

\usepackage[textsize=scriptsize,textwidth=3cm]{todonotes}
\usepackage{etoolbox}
\makeatletter
\patchcmd{\@addmarginpar}{\ifodd\c@page}{\ifodd\c@page\@tempcnta\m@ne}{}{}
\makeatother
\reversemarginpar

\usepackage{multicol}

\usepackage{comment}
\usepackage{amssymb}
\usepackage{textcomp}
\usepackage{subfigure}
\usepackage{tikz,pgf}
\usetikzlibrary{arrows,snakes,backgrounds,automata,shapes,patterns,decorations,backgrounds,plotmarks,calc}

\begingroup
\catcode`\~=11
\gdef\urltilde{\lower 0.6ex\hbox{~}}
\endgroup

 \newcommand{\B}{\mathcal{B}}

\newcommand{\G}{\mathcal{G}} \renewcommand{\H}{\mathcal{H}}

 \newcommand{\R}{\mathcal{R}}
\renewcommand{\S}{\mathcal{S}} \newcommand{\T}{\mathcal{T}}

\newcommand{\last}{\mathname{last}}

\newcommand{\nop}[1]{{}}

\newcommand{\defterm}[1]{\ul{\emph{#1}}}

\def\defeq{\ensuremath{:=}}

\newcommand{\ra}{\rightarrow}

\newcommand{\lora}{\longrightarrow}

\newcommand{\goto}[1]{\stackrel{#1}{\lora}}

\newcommand{\set}[1]{\{#1\}}                      % set

\newcommand{\card}[1]{|{#1}|}  % cardinality of a set

\newcommand{\tup}[1]{\langle #1\rangle}            % tuple

\newcommand{\Nat}{{\rm I\kern-.23em N}}

\newcommand{\Init}{\mathit{Init}}

\newcommand{\citea}[1]{\citeauthor{#1}~\cite{#1}}

\newcommand{\myi}{\emph{(i)}\xspace}
\newcommand{\myii}{\emph{(ii)}\xspace}

\newcommand{\Spec}{K}
\newcommand{\Specc}{\Spec_{\tup{\S,\T}}}
\newcommand{\Speccaprox}{{\hat{\Spec}}_{\tup{\S,\T}}}
\newcommand{\TSpec}{\supremal{\Spec}}
\newcommand{\TSpecc}{\supremal{\Specc}}
\newcommand{\TSpeccaprox}{\supremal{\Speccaprox}}

\newcommand{\Act}{\Sigma}
\newcommand{\BAct}{A}
\newcommand{\Mkd}{L_m}
\newcommand{\Gtd}{L}
\renewcommand{\sup}{V}
\newcommand{\Closure}[1]{\overline{#1}}

\newcommand{\Nondet}{\textsf {\small Succ}}
\newcommand{\Services}{\ensuremath{\mathsf{Indx}}}
\newcommand{\supremalC}{\mathname{sup}\Contr}
\newcommand{\Contr}{\mathsf{C}}

\newcommand{\Plant}{\G_{\tup{\S,\T}}}

\newcommand{\PlantApprox}{\hat{\G}_{\tup{\S,\T}}}

\newcommand{\CGDES}{cg_{\textsc{DES}}}
\newcommand{\CG}{cg}
\newcommand{\preceqnd}{\preceq_{\text{ND}}}

\newcommand{\lowsub}[1]%
{{\scriptsize\raisebox{.4\height}{$#1$}}}%

\newcommand{\supremal}[1]{{#1}\lowsub{\uparrow}}

\newcommand{\Targ}{\T}
\newcommand{\MaxTarg}{\T^*}
\newcommand{\Be}[1]{\B_#1}
\newcommand{\Behaviors}{\Be{1},\ldots,\B_{n}}

\newcommand{\word}{\texttt{word}}
\newcommand{\Comp}{P}

\tikzstyle{every initial by arrow}=[initial text=]
\tikzstyle{every state}=[fill=none,draw=black,text=black]
\tikzstyle{every state}=[fill=none,draw=black,text=black,inner sep=0pt,minimum
size=7mm]
\tikzstyle{every picture}=[->,>=stealth',shorten >=1pt,auto,node distance=2.5cm,
semithick]
\tikzstyle{sim}=[->,dotted]
\usetikzlibrary{shapes,snakes,fit}
\newcommand{\labelfig}[1]{\textcolor{blue}{\sc #1}}

\newcommand{\mathname}[1]{\ensuremath{\text{\textit{#1}}}}

\newcommand{\actionfont}[1]{\text{\textsc{#1}}}

\newcommand{\amine}{\actionfont{GoMine}}
\newcommand{\aload}{\actionfont{load}}
\newcommand{\adepo}{\actionfont{GoDepot}}
\newcommand{\aunload}{\actionfont{unload}}
\newcommand{\arepair}{\actionfont{repair}}
\newcommand{\adrill}{\actionfont{dig}}

\newcommand{\stateproj}{\mathname{st}}

\renewcommand{\mapsto}{\ra}

\usepackage[amsmath,thmmarks]{ntheorem}
\theoremseparator{.}

\theorembodyfont{\itshape}
\newtheorem{theorem}{Theorem}

\newtheorem{lemma}{Lemma}
\newtheorem{corollary}{Corollary}

\theorembodyfont{\normalfont}
\theoremsymbol{\ensuremath{\blacksquare}}
\newtheorem{definition}{Definition}

\theoremsymbol{\ensuremath{\square}}
\newtheorem{example}{Example}

\theorembodyfont{\normalfont}
\theoremstyle{nonumberplain}
\theoremsymbol{\ensuremath{\square}}
\theoremheaderfont{\normalfont\itshape}

\newtheorem{proof}{Proof}

\usepackage[numbers,compress]{natbib}

\title{Supervisory Control for Behavior Composition}

\author{
        Paolo~Felli \\ The University of Melbourne, Australia \\ paolo.felli@unimelb.edu.au \and
        Nitin~Yadav and Sebastian Sardina \\ RMIT University, Australia \\ \{name.surname\}@rmit.edu.au% <-this % stops a space
}
\date{}

\begin{document}

\twocolumn[
  \begin{@twocolumnfalse}
    \maketitle
 
\begin{abstract}
We relate behavior composition, a synthesis task studied in AI, to supervisory control theory from the discrete event systems field.
In particular, we show that realizing (i.e., implementing) a target behavior module (e.g., a house surveillance system) by suitably
coordinating a collection of available behaviors (e.g., automatic blinds, doors, lights, cameras, etc.) amounts to imposing a supervisor
onto a special discrete event system.
%, thus establishing the formal link between the two synthesis tasks.
%%
 Such a link allows us to leverage on the solid foundations and
 extensive work on discrete event systems, including borrowing tools and ideas from that field.
As evidence of that we show how simple it is to introduce preferences in the mapped framework.
\end{abstract}

\vspace{.1in}
   \end{@twocolumnfalse}
]

% \renewcommand{\citea}[1]{\cite{#1}}
%%%%%%%%%%%%%%%%%%%%%%%%%%%%%%%%%%%%%%%%%%%%%%%%%%%%%%%%%%%%%%%
\section{Introduction} %\label{sec:intro}
%%%%%%%%%%%%%%%%%%%%%%%%%%%%%%%%%%%%%%%%%%%%%%%%%%%%%%%%%%%%%%%

\newcommand{\GRAIL}{\textsf{\small GRAIL}\xspace}
\newcommand{\TCT}{\textsf{\small TCT}\xspace}
\newcommand{\STCT}{\textsf{\small STCT}\xspace}
\newcommand{\DESUMA}{\textsf{\small DESUMA}\xspace}
\newcommand{\SUPREMICA}{\textsf{\small SUPREMICA}\xspace}

In this paper, we formally relate two automatic synthesis tasks, namely, \emph{behavior composition}, as studied within the AI community
(e.g., ~\cite{BerardiCDGLM:IJCIS05,DeGiacomoS:IJCAI07,StroederPagnucco:IJCAI09,DeGiacomoPatriziSardina:AIJ13}) and \emph{supervisory control} in discrete event systems~\cite{WonhamRamadge:SIAMJCO87,Ramadge:SIAM87,RamadgeWonhamDES:IEEE89,Cassandras:BOOK06-DES}.
By doing that, we aim at facilitating the awareness and cross-fertilization between the two different communities and techniques available.

The composition problem involves automatically ``realizing'' (i.e., implementing) a desired, though virtual, \emph{target} behavior module by suitably coordinating the execution of a set of concrete \emph{available} behavior modules.
From an AI perspective, a behavior refers to the abstract operational model of a device or program, generally represented as a nondeterministic transition system.
For instance, one may be interested in implementing a house entertainment system by making use of various devices installed, such as game/music consoles, TVs, lights, etc.

Supervisory Control, on the other hand, is the task of automatically synthesizing ``supervisors'' that \emph{restrict} the behavior of a ``plant'', i.e. a discrete event system (DES) which is assumed to spontaneously generate events, such that a given specification is fulfilled.
DES models a wide spectrum of physical systems, including manufacturing, traffic, logistics, and database systems.
In Supervisory Control Theory (SCT), an automaton $\G$---known as \emph{``the plant''}---is used to model both controllable and uncontrollable behaviors of a given DES. The assumption is that the overall behavior of $\G$ is \emph{not} satisfactory and must be controlled. To that end, a so-called  \emph{supervisor} $\sup$ is imposed on $\G$ so as to meet a given specification on event orderings and legality of states. %%
Supervisors observe (some of) the events executed by $\G$ and can \emph{disable} those that are controllable in order to guarantee a given specification.

Both behavior composition and supervisory control can be seen as \emph{generalized} forms of automated planning tasks~\cite{Srivastava:GENPLAN-TR08}.
Rather than building (linear) plans to bring about an (achievement) goal, the aim is to keep the system in certain ``good'' states. Since we are to build controllers meant to run continuously, solutions generally include ``loops.'' Moreover, in contrast with classical planning, the domains are nondeterministic in nature, which relates  to FOND planning and strong-cyclic notions of plans, as shown in~\cite{RamirezYadavSardina:ICAPS13}.

To build a bridge between the two problems and communities, this article provides the following  technical contributions:
\begin{itemize}
  \item A formal, provably correct (Theorems~\ref{theo:soundness} and \ref{theo:completness}) reduction of the AI behavior composition problem to the problem of controlling a regular language on a particular DES plant (which we call ``composition plant''). 
  
  \item A technique to extract the controller generator---the universal solution for a composition problem---from a supervisor of the DES composition plant. We show the technique is correct (Theorem~\ref{theo:extract_cg_correct}), optimal w.r.t. computational complexity (in the worst case; Theorem~\ref{theo:extract_cg_complexity}), and realizable using existing off-the-shelf SCT tools.

  \item An approach to DES-based behavior composition \emph{approximation} for the special case of deterministic system (as it is the case, for example, in web-service composition). This is appealing when no composition solution exists and one hence looks  for ``the best'' possible controller.
\end{itemize}

The motivations behind linking behavior composition to supervisory control theory are threefold.
First, supervisory control theory has rigorous \emph{foundations rooted in formal
languages}. It was first developed by~\cite{RamadgeWonhamDES:IEEE89} and others in the 80's and then further strengthened by many other researchers w.r.t. both theory and application~(see, e.g., \cite{Cassandras:BOOK06-DES} for a broad overview of the field).
Recasting the composition task as the supervision of a DES provides us with a solid foundation for studying composition.
Second, computational properties for supervisor synthesis have been substantially studied and tools for supervisor synthesis are available, including \TCT/\STCT~\cite{ZhangWonham:SSCDES01}, \GRAIL~\cite{Reiser:WODES06-GRAIL}, \DESUMA~\cite{Ricker:WODES06-DESUMA}, and \SUPREMICA~\cite{Akesson:2003-SUPERMICA}. Thus, we can apply very different techniques for solving the composition problem than those already available within the AI literature (e.g., PDL satisfiability~\cite{DeGiacomoS:IJCAI07}, direct search~\cite{StroederPagnucco:IJCAI09}, LTL/ATL synthesis~\cite{LustigVardi:FOSSACS09,DeGiacomoFelli:AAMAS10}, and computation of special kind of simulation relations~\cite{DeGiacomoPatriziSardina:AIJ13,BerardiCDGP:IJFCS07}).
Finally, once linked, we expect cross-fertilization between AI-style composition and DES supervisory theory.
To that end, for instance, we demonstrate here how DES-based composition can directly and naturally handle constraints over a composition task.
Indeed, one may look at importing powerful notions, and corresponding techniques, common in SCT, such as hierarchical and decentralized supervision, maximal controllability, and tolerance supervision.

%%%%%%%%%%%%%%%%%%%%%%%%%%%%%%%%%%%%%%%%%%%%%%%%%%%%%%%%%%%%%%%
\section{Preliminaries} 
\label{sec:preliminaries}
%%%%%%%%%%%%%%%%%%%%%%%%%%%%%%%%%%%%%%%%%%%%%%%%%%%%%%%%%%%%%%%

In this section, we very briefly review the required background to understand the rest of the paper.

\subsection{The Behavior Composition Problem} \label{sec:preliminaries_bc}

The behavior composition problem has been recently much studied in the web-services and AI literature~\cite{BerardiCDGHM:VLDB05,BerardiCDGP:IJFCS07,DeGiacomoPatriziSardina:AIJ13,StroederPagnucco:IJCAI09,LustigVardi:FOSSACS09}.  
The problem amounts to synthesising a \emph{controller} that is able to ``realize'' (i.e., implement) a desired, but nonexistent, \emph{target behavior} module, by suitably coordinating a set of \emph{available behavior} modules.
In what follows, we mostly follow the model detailed
in~\cite{DeGiacomoPatriziSardina:AIJ13}.\footnote{For legibility, and without loss of generality, we
  leave out the so-called shared environment, used to model action's preconditions. 
}

Generally speaking, behaviors represent the operational logic of a device or program, and they are modeled using, possibly nondeterministic, finite transition systems.
Nondeterminism is used to express the fact that one may have have incomplete information about a behavior's logic.

\begin{example}\label{ex:mining}
Consider the example depicted in Figure~\ref{fig:ambientIntel-system}.
% %%
Target $\T$ encapsulates the desired functionality of a mining
system, which allows to unboundedly extract minerals from the ground ($\adrill$), move (some transportation vehicle) to the extraction area ($\amine$) and stock the loose
materials to a certain deposit (action sequence $\aload$--$\adepo$--$\aunload$). Finally, routine repairs are performed ($\arepair$). 
 At every step, the user requests an action compatible with this
 specification, and a (good) controller should guarantee that it can
 fulfill such request by delegating the action to one of the three
machines actually available in the mine: a dumper truck
 (initially at the depot), a loader, and an old excavator (which are
 initially at the mine).
Note that such machines may be nondeterministic: the truck can break
down while trying to reach the mining area due to terrain
conditions, whereas the loader might need to perform repairs after
unloading. The excavator (which can not move from the
extraction area) is instead deterministic; however, it is mainly
intended for digging, not for loading. As a result, whenever it
performs a $\aload$ action, it needs to be repaired before being able
to load again.
\end{example}

Technically, a \defterm{behavior} is a tuple $\B=
\tup{B,\BAct,b_0,\delta}$ where: 
\begin{itemize}\itemsep=0pt
\item $B$ is the finite set of states; 
\item $\BAct$ is the set of actions; 
\item $b_0 \in B$ is the initial state; and
\item $\delta \subseteq B\times \BAct \times B$ is $\B$'s
  (nondeterministic) transition relation:
  $\tup{b,a,b'}\in\delta$ denotes that action $a$ executed
  in behavior state $b$ may lead the behavior to successor state $b'$.
\end{itemize}

We also use alternative notations for the transition relation, by freely exchanging the notations $\tup{b,a,b'}\in\delta$, $b \goto{a} b'$ in $\B$, and $b'\in\delta(b,a )$.
If, for any state $b$ and action $a$, there exists a unique successor state $b'\in\delta(b,a)$, then we say that $\B$ is \emph{deterministic}, in the sense that its successor state is uniquely determined by the current state and the chosen action, thus we write $b'=\delta(b,a)$.
A \defterm{trace} of $\B$ is the possibly infinite sequence $\tau=b_0\goto{a_1}b_1\goto{a_2}\cdots$ such that $b_{i+1} \in \delta(b_\ell,a_{\ell+1})$, for $\ell\geq 0$. A \defterm{history} is a finite trace.

An (available) \defterm{system} is a tuple $\S = \tup{\B_1,\ldots,\B_n}$, where each $\B_i =\tup{B_i,\BAct,b_{0i},\delta_i}$ is referred to as an \emph{available behavior} in $\S$ (over shared actions $\BAct$). 
The joint asynchronous execution of  $\S$ is captured by the so-called \defterm{enacted system}
$\B_\S = \tup{B_\S,\BAct,\vec{b}_0,\delta_\S }$,\footnote{The term
  ``enacted'' is due to the fact that, in the full composition
  setting, all behaviors are meant to be run within a shared
  environment (which, without loss of generality, we left out in this work for simplicity).} where:
\begin{itemize}
\item $B_\S = B_1 \times \ldots \times B_n$ is the set of system states of $\B_\S$ (given $\vec{b} = \tup{b_1,\ldots,b_n} \in B_\S$, we denote $\stateproj_i(\vec{b}) = b_i$ for all $i \in \set{1,\ldots,n}$); 
\item $\vec{b}_0 = \tup{b_{01},\ldots,b_{0n}} \in B_\S$ is the initial state of $\B_\S$; and
\item $\tup{\vec{b},a,j,\vec{b}'} \in \delta_\S$ \emph{iff} 
$\stateproj_j (\vec{b}) \goto{a} \stateproj_j (\vec{b'})$ in $\B_j$ and 
$\stateproj_i(b) = \stateproj_i(\vec{b'})$,  for all $i \in \set{1,\ldots,n}\setminus \set{j}$.
\end{itemize}
A \defterm{system history} is a straightforward generalization of behavior histories to an available system $\B_\S$, that is, a sequence 
of the form $h =
\vec{b_0}\goto{a_1,j_1}\vec{b_1}\goto{a_2,j_2} \cdots \goto{a_\ell,j_\ell} \vec{b_\ell}$.
We denote with $\last(h)$  the last state $\vec{b}_\ell$ of $h$ and with $\H$ the set of all system histories.
Finally, the \defterm{target behaviour} module is just a deterministic
behavior $\Targ = \tup{T,\BAct_t,t_0,\delta_t}$. For clarity, we denote
its states by $t$ instead of $b$.
Hence, a \defterm{target trace} is a, possibly infinite, sequence $\tau = t_0 \goto{a_1} t_1 \goto{a_2} \cdots$, such that $t_i \goto{a_{i+1}} t_{i+1}$ in $\T$, for all $i \geq 0$.

%%%%%%%%%%%%%%%% EXAMPLE %%%%%%%%%%
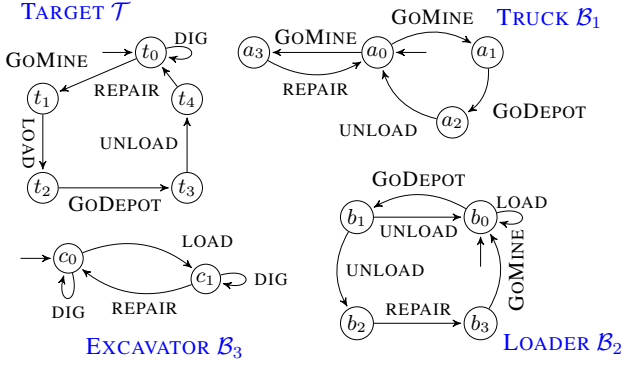
\begin{figure}[!t]
\begin{center}
\resizebox{\columnwidth}{!}{\begin{tikzpicture}[->,thin,node distance=1.8cm,double distance=2pt]
\tikzstyle{every state}=[circle,fill=none,draw=black,text=black,inner
sep=1pt,minimum size=4mm,font=\small]

\tikzset{ActionStyle/.style = {font=\small}}

%%% TARGET
\begin{scope}[shift={(0cm,0cm)}]

\node[initial left,state]	(t0)                    {$t_0$};
\node[state]	(t1)   [below left of=t0,shift={(-0.3cm,0.6cm)}]              {$t_1$};
\node[state]	(t2)   [below of=t1,shift={(0cm,0.5cm)}]              {$t_2$};
\node[state]	(t3)   [right of=t2,shift={(.3cm,0cm)}]              {$t_3$};
\node[state]	(t4)   [above of=t3,yshift=-.2in]              {$t_4$};

\path
(t0) edge[loop right]  node[above,ActionStyle] {$\adrill$}    (t0)	
 (t0) edge[]  node[swap,ActionStyle] {$\amine$}    (t1)	
(t1) edge[]  node[sloped,below,ActionStyle] {$\aload$}    (t2)	
(t2) edge[]  node[ActionStyle,below] {$\adepo$}    (t3)	
(t3) edge[]  node[ActionStyle] {$\aunload$}    (t4)	
(t4) edge[]  node[ActionStyle] {$\arepair$}    (t0)	
;

\node (name)[above of=t0,shift={(-1.1cm,-1.2cm)}]	
	{\labelfig{Target $\T$}};
\end{scope}

%%% TRUCK
\begin{scope}[shift={(1.5cm,0cm)}]
\node[state]                (a3) []      {$a_3$};
\node[initial right,state]	(a0)    [right of=a3]               {$a_0$};
\node[state]    			(a1) [right of=a0,xshift=-.2cm] {$a_1$};
\node[state]                (a2) [below left of=a1,xshift=.3in,yshift=.1in]      {$a_2$};

\path	
(a0) edge[bend left]  node[ActionStyle] {$\amine$}    (a1)
(a1) edge[bend left]  node[pos=.5,ActionStyle] {$\adepo$}    (a2)
(a2) edge[bend left]  node[very near start,ActionStyle] {$\aunload$} (a0)
(a0) edge[]  node[swap,ActionStyle] {$\amine$}   (a3)
(a3) edge[bend right]  node[swap,ActionStyle] {$\arepair$}   (a0)
;

\node (name)[right of=a1,shift={(-.9cm,.5cm)}] {\labelfig{Truck $\B_{1}$}};
\end{scope}

%%% LOADER
\begin{scope}[shift={(3cm,-2.4cm)}]

\node[state]    			(b1) [] 		{$b_1$};
\node[initial below,state]	(b0)  [right of=b1]   {$b_0$};
\node[state]                (b2) [below of=b1,yshift=.1in]      {$b_2$};
\node[state]                (b3) [right of=b2]      {$b_3$};

\path
(b0) edge[loop right]  node[above,ActionStyle] {$\aload$}    (b0)	
(b0) edge[bend right]  node[swap,ActionStyle] {$\adepo$}    (b1)
(b1) edge[]  node[swap,ActionStyle] {$\aunload$}    (b0)
(b1) edge[bend right]  node[ActionStyle] {$\aunload$}    (b2)
(b2) edge[]  node[ActionStyle] {$\arepair$}    (b3)
(b3) edge[bend right]  node[sloped, below,ActionStyle] {$\amine$}    (b0)
;

\node (name)[below of=b3,shift={(1.2cm,1.5cm)}]	
	{\labelfig{Loader $\B_{2}$}};
\end{scope}

%%EXCAVATOR
\begin{scope}[shift={(-1.2cm,-3cm)},node distance=2cm,auto]
\node[initial left,state]	(c0)                    {$c_0$};
\node[state]    			(c1) [right of=c0,shift={(0cm,-.3cm)}] 		{$c_1$};

\path	
(c0) edge[loop below]  node[ActionStyle] {$\adrill$}    (c0)
(c1) edge[loop right]  node[ActionStyle] {$\adrill$}    (c1)
 (c0) edge[bend left]  node[ActionStyle,pos=.8] {$\aload$}    (c1)
(c1) edge[bend left]  node[pos=0.4,ActionStyle] {$\arepair$}    (c0)
;

\node (name)[below of=c0,shift={(1.4cm,.7cm)}]	
	{\labelfig{Excavator $\B_{3}$}};
\end{scope}

\end{tikzpicture}}
\end{center}
\caption{A mining system with three available machines.}
\label{fig:ambientIntel-system}
% \vspace{-.15in}
\end{figure} 
%%%%%%%%%%%%%%%%%%%%%%%%%%%%%%%%%%

A so-called \defterm{controller} is a function of the form 
\[
\Comp : \H \times \BAct_t \ra \set{ 1,\ldots,n}
\]
that takes a history (i.e., a run) of the system and the next action request, and outputs the index of the available behavior where the action is to be  delegated. 
The composition task then amounts to whether there exists (and if so, how to compute it) a controller $P$ such that the target behavior is ``realized,'' that is, it looks as if the target module is being executed.

Roughly speaking, a controller realizes a target module if it is always able to further extend all the system traces (by prescribing adequate action delegations), no matter how the available behaviors happen to evolve (after each step).
To capture this, one first define the set $\H_{P,\tau}$ of all
$(P,\tau)$-induced system histories, that is, those system histories
(i.e., histories of enacted system $\B_\S$) with action requests as
per target trace $\tau$ and action delegations as per controller $P$.%
\footnote{The set of $(P,\tau)$-induced system histories $\H_{P,\tau}$ is defined as $\H_{P,\tau} = \bigcup_k \H_{P,\tau}^k$, where $\H_{\Comp,\tau}^0 = \set{\vec{b}_0}$ (all
behaviors are in their initial state) and $\H_{\Comp,\tau}^{k+1}$ is the set of all possible histories of length $k$+1 obtained by applying $\Comp$ to a history in $\H_{\Comp,\tau}^{k}$. See~\cite{DeGiacomoPatriziSardina:AIJ13} for details.}

\begin{definition}[\cite{DeGiacomoS:IJCAI07,DeGiacomoPatriziSardina:AIJ13}]
Controller $\Comp$ \defterm{realizes a target trace} $\tau$, as above,  in a
system $\S$ if for all $(P,\tau)$-induced system histories $h\in
\H_{P,\tau}$ with $\card{h} < \card{\tau}$, there exists an enacted
system (successor) state $\vec{b}_{\card{h}+1}$ such that
$last(h)\goto{a_{\card{h}+1,j_h}}\vec{b}_{\card{h}+1}$ is in $\B_\S$
with $j_h=\Comp(h,a_{\card{h}+1})$.
A controller $\Comp$ \defterm{realizes a target behavior} $\Targ$ (in a system $\S$) iff it realizes all the traces of $\Targ$.
\end{definition}
 
The existence requirement of a system successor state $\vec{b}_{\card{h}+1}$ implies that  delegation $\Comp(h,a_{\card{h}+1})$ (of action request $a_{\card{h}+1}$ in system history $h$) is legal (i.e., is able to extend the current history $h$). 
Whenever a controller realizes a target behavior $\Targ$ in a system $\S$, we say that such controller is an \defterm{exact compositions} of $\Targ$ in $\S$. 
It is not difficult to see that there is indeed an exact composition for the example in Figure~\ref{fig:ambientIntel-system}: all actions requested as per the target logic will \emph{always} be fulfilled (i.e., delegated to an available behavior) by the controller, forever.

As one may expect, checking the existence of an exact composition is EXPTIME-complete~\cite{ 
DeGiacomoS:IJCAI07}, as it resembles conditional planning under full observability~\cite{Rintanen:ICAPS04}.
Interestingly, by revisiting a certain stream of work in service composition area, the
technique devised in \citea{DeGiacomoPatriziSardina:AIJ13} allows to
synthesize a sort of meta-controller, called \defterm{controller generator} (CG) representing \emph{all} possible compositions. 
Concretely, a CG is a function 
\[
	\CG : B_\S \times \BAct_t \ra 2^{\set{1,\ldots,n}}
\]
that, given a system state and a target action $a$, returns a \emph{set} of behavior indexes to which the requested action $a$ may be legally delegated.
A controller generator $\CG$ \emph{generates} a concrete controller $\Comp$ iff $\Comp(h,a) \in \CG(last(h),a)$ for any system history $h$ and action $a$ compatible with $P$ and the target logic, respectively. 
The CG is unique and finite, and represents a flexible and robust
solution concept to the composition problem~\cite{DeGiacomoPatriziSardina:AIJ13}.\footnote{Note that there is a potentially  uncountable set of composition controllers. To see this, consider any subset $E \subseteq \Nat$ of natural numbers and define controller $C_E$ to delegate to behavior $\B_1$ if the length $n$ of the current history is in $E$ and to $\B_2$ otherwise. There is clearly one controller for each subset of $\Nat$, and thus there is an uncountable number of controllers.}

We close by noting that \citea{DeGiacomoPatriziSardina:AIJ13}'s technique is directly based on the idea that a composition amounts to a module that coordinates the concurrent execution of the available behaviours so as to ``mimic'' the desired target behaviour. 
This ``mimicking'' is captured through the formal notion of
simulation~\cite{Miln71}, suitably adapted to deal with
nondeterministic behaviors. 
Intuitively, a
behavior $\B_1$ ``simulates'' another behavior $\B_2$, denoted
$\B_2 \preceq \B_1$, if $\B_1$  is able to always \textit{match} all of
$\B_2$'s moves. 

%\begin{definition}
Formally, given two behaviors $\B_1=
\tup{B_1,\BAct,b_{01},\delta_1}$ and $\B_2=
\tup{B_2,\BAct,b_{02},\delta_2}$, a simulation relation of
$\B_1$ by $\B_2$ is a relation $R \subseteq B_1 \times B_2$ such that
$\tup{b_1,b_2}\in R$ implies that for any action $a\in A$ and
transition $b_1\goto{a}b_1'$ in $\B_1$, there exists a
transition $b_2\goto{a}b_2'$ in $\B_2$.
%\end{definition}

Importantly, \citea{DeGiacomoPatriziSardina:AIJ13} defined a
so-called (greatest) \defterm{\textsc{nd}-simulation} relation (\textsc{nd} stands for
nondeterministic) between (the states of) the target behavior $\Targ$
and (the states of) the enacted system $\B_\S$, denoted $\preceqnd$.

\begin{definition}
Consider a target $\Targ = \tup{T,\BAct_t,t_0,\delta_t}$ and the enacted system $\B_\S = \tup{B_\S,\BAct,\vec{b}_0,\delta_\S }$. An
\textsc{nd}-simulation relation of $\Targ$ by $\B_\S$ is a relation $R \subseteq
T \times B_\S$ such that $\tup{t,\vec{b}}\in R$ implies that for all
actions $a\in A$ there exists an index $j\in\set{1,\ldots,n}$ such
that for all transitions $t\goto{a}t'$ in $\Targ$ \myi there exists a transition $\vec{b}\goto{a,j}\vec{b}'$ in $\B_\S$, and \myii for all $\vec{b}\goto{a,j}\vec{b}'$ in $\B_\S$ we have $\tup{t',\vec{b}'}\in R$.
\end{definition}

The following result holds.

\begin{theorem}[\cite{BerardiCDGP:IJFCS07,DeGiacomoPatriziSardina:AIJ13}]
\label{th:sim_initial_states}
There exists a composition controller of a target behavior $\T$ in an available
system $\S$ if and only if $t_0 \preceqnd \vec{b_0}$ (where $t_0$ and $\vec{b_0}$ are $\T$'s and $\B_\S$'s initial states).
\end{theorem}

In this paper we are indeed interested in synthesising controller generators, and not just single composition controllers.
However, instead of building an \textsc{nd}-simulation relation, we aim at extracting the controller generator by leveraging on existing techniques in Supervisory Control Theory.

\subsection{Supervisory Control in Discrete Event Systems} \label{sec:preliminaries_des}

Discrete event systems range across a wide variety of physical systems that arise in technology (e.g., manufacturing and logistic systems, DBMSs, communication protocols and networks, etc.), whose processes are discrete (in time and state space), event-driven, and nondeterministic~\cite{Cassandras:BOOK06-DES}.
Generally speaking, Supervisory Control Theory is concerned with the \emph{controllability} of the sequences (or strings/words) of events that such processes/systems---commonly referred as the \emph{plant}---may generate~\cite{RamadgeWonhamDES:IEEE89}.
As standard in formal languages, a \defterm{language} $L$ over a set $\Act$ is any set $L \subseteq \Act^*$, and $\epsilon \in \Act^*$ denotes the empty string. 
The \defterm{prefix-closure} of a language $L$, denoted by $\Closure{L}$, is the language of all prefixes of words in $L$, that is, $w \in \Closure{L}$ if and only if $w\concat w' \in L$, for some $w' \in \Sigma^*$ ($w \concat w'$ denotes the concatenation of words $w$ and $w'$).
A language $L$ is  \defterm{closed} if $L = \Closure{L}$.

In SCT, the plant is viewed as a generator of the language of string of events characterizing its processes. 
Formally, a \defterm{generator} is a deterministic finite-state machine $\G=\tup{\Act,G,g_0,\gamma,G_m}$, where $\Act$ is the finite
alphabet of events; 
$G$ is a finite set of states;
$g_0\in G$ is the initial state;
$\gamma : G \times \Act \to G$ is the transition function; and
$G_m\subseteq G$ is the set of marked states.
We generalize transition function $\gamma$ to words as
follows: $\gamma :  G\times \Act^*\ra G$ is such that  $\gamma(g,\epsilon) = g$ and $\gamma(g,w\concat \sigma) = \gamma(\gamma(g,w),\sigma)$, with $w\in \Act^*$ and $\sigma\in\Act$.
We say that a state $g\in G$ is \defterm{reachable} if $g = \gamma(g_0,w)$ for some word $w\in \Act^*$.
% ; whereas  $g$ is \defterm{coreachable} if $\gamma(g,w)\in G_m$, namely, a marked state is reachable from state $g$.
%%
Finally, given two words $w_1$,$w_2 \in \Act^*$, $w_1 > w_2$ iff $w_1 = w_2 \concat w$, for some  $w\neq\epsilon$.

The \defterm{language generated} by generator $\G$ is $\Gtd(\G) = \set{ w\in \Act^* \mid \gamma(g_0,w) \text{ is defined} }$, whereas the \defterm{marked} language of $\G$ is $\Mkd(\G) = \set{ w \in \Gtd(\G) \mid \gamma(g_0,w) \in G_m }$. Words in the former language stand for, possibly partial, operations or tasks, while words in the marked language represent the completion of some operations or tasks.
Note that $\Gtd(\G)$ is always closed, but $\Mkd(\G)$ may not be.

\newcommand{\gOn}{\mathname{on}}
\newcommand{\gOperate}{\mathname{operate}}
\newcommand{\gBreak}{\mathname{break}}
\newcommand{\gDismantle}{\mathname{dismantle}}
\newcommand{\gRepair}{\mathname{repair}}
\newcommand{\gOff}{\mathname{off}}
\newcommand{\gReassemble}{\mathname{reassemble}}

Central to generators is the distinction between those events that are controllable and those that they are not.
Technically, the generator's alphabet is partitioned into \defterm{controllable} ($\Act_c$) and \defterm{uncontrollable} ($\Act_u$) events, that is, $\Act = \Act_c \cup \Act_u$, where $\Act_c\cap\Act_u = \emptyset$.
All events may occur only when \emph{enabled}. Whereas controllable events may be enabled or disabled, uncontrollable events are assumed to be always enabled.

\begin{example}\label{ex:generator}
Figure \ref{fig:generator} shows a generator $\G$ modeling a generic industrial machine. 
The machine can be started (event $\gOn$), then repeatedly operated (event $\gOperate$), finally stopped (event $\gOff$). All these events are \emph{controllable},  in that their occurrence is in the hand of the machine's user. 
While machine is in state $1$, the machine may unexpectedly break down, signaled by the occurrence of event $\gBreak$. The occurrence of such event is however outside the control of the machine's user---the event is \emph{uncontrollable}.
When the machine breaks down, it ought to be either repaired or dismantled (events $\gRepair$ and $\gDismantle$, resp.), both within the control of the user.

As a generator, $\G$ produces words---those in $L(\G)$---representing the possible runs (i.e., executions) of the machine being modelled. 
In particular, the machine's marked language $\Mkd(\G) $ is equivalent to the regular expression $(\gOn \concat ( \gOperate \mid ( \gBreak \concat \gRepair ))^* \concat \gOff )^* $ that 
corresponds to those sequences of events that leave the machine in state $0$.
Words in $\Mkd(\G)$ are said to be ``marked,'' in that they are judged ``complete,'' and therefore ``good'' (in the eye of the machine's designer).
\end{example}

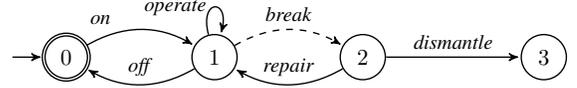
\begin{figure}[t]
\begin{center}
\resizebox{.9\columnwidth}{!}{\begin{tikzpicture}
\tikzset{ActionStyle/.style = {font=\small}}

\begin{scope}[shift={(0cm,0cm)},node distance=2.3cm,auto]
\node[double,state,initial]	(0)                   {$0$};
\node[state]    	(1) [right of=0]   {$1$};
\node[state]    	(2) [right of=1]   {$2$};
\node[state]    	(3) [right of=2,xshift=.2in]   {$3$};

\path	
(0) edge[pos=.3,bend left]  node[ActionStyle] {$\gOn$}  (1)
(1) edge[dashed,bend left]  node[ActionStyle] {$\gBreak$}  (2)
(1) edge[loop above]  node[left,ActionStyle] {$\gOperate$}  (1)
(2) edge[bend left]  node[above,ActionStyle] {$\gRepair$}  (1)
(1) edge[bend left]  node[above,ActionStyle] {$\gOff$}  (0)
(2) edge[]  node[ActionStyle] {$\gDismantle$}  (3)
;

\end{scope}

\end{tikzpicture}}
\caption{A generator modeling a simple machine.}
\label{fig:generator}
\end{center}
% \vspace{-.25in}
\end{figure}

As expected, the overarching idea in SCT is to check whether one is able to guarantee certain specified (good) behavior of the device being modeled by a generator, and if so, how.
A \defterm{specification} for a generator plant  $\G$ is a language $\Spec\subseteq\Gtd(\G)$. 
We are now prepared to formally introduce the key notion of controllability in SCT.

\begin{definition}\label{def:controllable}
A specification $\Spec$ is \defterm{controllable} in generator-plant
$\G$ if and only if $\Closure{\Spec} \cdot \Act_u  \cap  \Gtd(\G)  \subseteq \Closure{\Spec}$.
\end{definition}

That is, every prefix of $\Spec$ immediately followed by a legal uncontrollable event (i.e., one compatible with $\G$) can be extended to a word in the specification itself. 
Intuitively, $\Spec$ is controllable if it is not possible to be
``pushed" outside of it, regardless of potential uncontrollable events.

\begin{example}\label{ex:generator2}
Consider the specification $\Spec_1 = L((\gOn \concat \gOperate^* \concat
\gOff)^*)$ requiring that the machine from Example~\ref{ex:generator} will always function without break downs. 
Clearly, such specification is \emph{not} controllable: there exists an uncontrollable event ($\gBreak$) that can violate it.
In fact, any word $w'=w\concat \gOn$ in $\Gtd(\G)$ (therefore any such $w' \in \Closure{\Spec_1}$) can be extended with the uncontrollable event $\gBreak \in \Sigma_u$, resulting in word $w' \concat \gBreak$ not meeting the specification, that is, $w' \concat \gBreak \not\in \Closure{\Spec_1}$.

Consider alternative specification $\Spec_2\! = \!L((\gOn\! \mid\! \gOff \mid \gOperate \mid \gBreak \mid \gRepair)^*)$, which %prescribes that the machine will never be dismantled
prohibits that the machine be dismantled. Such specification is indeed controllable: every time the machine breaks down, one needs to repair it.
Note how a very concrete process specification is embedded in $\Spec_1$, whereas $\Spec_2$ is more abstract, in that it does not describe a process but rather compactly captures a set of (good) processes.
\end{example}

The next step is to define what it means for a generator to be ``supervised'' in order to achieve certain behavior. The idea is that one---the supervisor---can disable certain controllable events to achieve a desired behavior.
Technically, a \defterm{supervisor} for a plant $\G$ is a function of the form 
\[
\sup : \Gtd(\G) \mapsto \set{\Act_{e} \mid  \Act_{e} \in 2^{\Act},\; \Act_u \subseteq \Act_{e} }
\]
that outputs, for each word in $\Gtd(\G)$, the set of events that are
enabled (i.e., allowed) next. 
Notice that uncontrollable events are always enabled.
A plant $\G$ under supervisor $V$ yields the \defterm{controlled system} $\sup/\G$ whose generated and marked languages are defined as follows:
\begin{align*}
\Gtd(\sup/\G) & = 
	\set{w\concat\sigma \in \Gtd(\G) \mid\ w \in \Gtd(\sup/\G),\; \sigma \in \sup(w)} \cup \set{\epsilon}; \\[0.75ex]
\Mkd(\sup/\G) & = \Gtd(\sup/\G)\cap \Mkd(\G).
\end{align*}

Informally, $\Gtd(\sup/\G)$ represents all processes that plant $\G$
may yield while supervised by $\sup$, whereas $\Mkd(\sup/\G)$ stands
for the subset 
%of those processes 
that are, in some sense, ``complete''.

A key result in SCT states that being able to control a (closed) specification in a plant amounts to finding a supervisor for such specification.

\begin{theorem}[\cite{Wonham:TR12-SCT}]\label{theo:controllability_theorem}
 Let $\G$ be a generator and  $\Spec \subseteq  \Gtd(\G)$ be a closed and non-empty specification.
%%
%Then, 
There exists a supervisor $\sup$ such that $\Gtd(\sup/\G)=\Spec$
\emph{iff} $\Spec$ is controllable in $\G$.
\end{theorem}

In many settings, one would further aim to control the language representing \emph{complete} processes, that is, the marked fragment of the plant.
In such cases, one shall focus on supervisors that can always drive the plant's execution towards the generation of words in the marked (supervised) language. 
Technically, supervisor $V$ is \defterm{nonblocking} in plant $\G$ if $\Gtd(\sup/\G) = \Closure{\Mkd(\sup/\G)}$.
This means that the strings in the supervised language $\Gtd(\sup/\G)$ are prefixes of marked supervised language $\Closure{\Mkd(\sup/\G)}$, and therefore they can always be potentially extended into a complete marked string.

\begin{example}\label{ex:generator3}
Consider a specification $\Spec_3$ for the generator in Figure~\ref{fig:generator} stating that one should $\gDismantle$ the machine when it breaks after exactly $n$ number of repairs. % after break-downs. 
Concretely, $w \in K_3$ iff either $w$ does not mention $\gDismantle$ and mentions $\gRepair$ less than $n$ times, or $w = w' \concat \gBreak \concat\gDismantle$ and $w'$ mentions $\gRepair$ exactly $n$ times but does not mention $\gDismantle$. 

Specification $K_3 \cap L(\G)$ is controllable, as there exists a supervisor $\sup$ that disables
event $\gDismantle$ in any run/word containing less than $n$ break-down events, and then enables it while disabling event $\gRepair$.
In particular, $\Gtd(\sup/\G)=\Spec_3 \cap L(\G)$ (see some words in $K_3$ may never arise in the plant $\G$).

Notice, however, that such supervisor $\sup$ is \emph{not} nonblocking: any string ending with the event $\gDismantle$ can not be extended to a marked string. If, instead, state $3$ were
marked (or, say, the machine featured a controllable event $\gReassemble$ from state $3$ to state $2$), then the same supervisor would be nonblocking, with $ \Closure{\Mkd(\sup/\G)}=\Spec_3 \cap L(\G)$.
\end{example}

Now, when a specification $K$ is \emph{not} (guaranteed to be) controllable, one then looks for controlling the ``largest" (in terms of set inclusion) possible sublanguage of $K$. 
Interestingly, such sublanguage, called the \defterm{supremal controllable sublanguage} of $K$ and denoted $\supremalC(\Spec)$, does exist and is in fact unique~\cite{WonhamRamadge:SIAMJCO87}.

Putting it all together, in SCT, we are generally interested in (controlling) the $K$'s sublanguage $\TSpec = \supremalC(\Spec\cap \Mkd(\G))$, that is, the supremal marked specification.
It turns out that, under a plausible assumption, a supervisor does exist for non-empty \nolinebreak$\TSpec$.

\begin{theorem}[\cite{Wonham:TR12-SCT}]\label{theo:supNSC}
If $\Closure{\Spec} \cap \Mkd(\G) \subseteq \Spec$ and $\TSpec \not= \emptyset$, there exists a nonblocking supervisor $\sup$ for $\G$ s.t. $\Mkd(\sup/\G) = \TSpec$.
\end{theorem}

The assumption that $\Closure{\Spec} \cap \Mkd(\G) \subseteq \Spec$  states that initial specification $K$ is \emph{closed under marked-prefixes}: every prefix from $K$ representing a complete process is part of $K$.
Theorem~\ref{theo:supNSC} will play a key role in our results.

\begin{example}\label{ex:generator4}
Consider again $\Spec_1 = L((\gOn \concat \gOperate^* \concat \gOff)^*)$ from Example \ref{ex:generator2}.  
Its supremal marked specification is $\supremal{\Spec_1} = \set{\epsilon}$, that is, the sole empty string. 
This is because as soon as the event $on$ is enabled and the machine moves to $1$, the uncontrollable event $\gBreak$ may occur, thus violating $\Spec_1$. Therefore, any word leading to state $1$ can not be in $\supremal{\Spec_1}$. 
\end{example}

%%%%%%%%%%%%%%%%%%%%%%%%%%%%%%%%%%%%%%%%%%%%%%%%%%%%%%%%%%%%%%%
\section{DES-based Behavior Composition}\label{sec:DEScomposition}
%%%%%%%%%%%%%%%%%%%%%%%%%%%%%%%%%%%%%%%%%%%%%%%%%%%%%%%%%%%%%%%
\newcommand{\req}{\mathsf{req}}
\newcommand{\idx}{\mathsf{idx}}

In this section we show how to relate the notion of a composition controller in behavior composition to that of an adequate supervisor in discrete event system. 
After all, their operational requirements are similar, namely, to take decisions in a step-by-step fashion in order to keep the system evolutions in a restricted set of ``good'' traces. 
Their differences can be summarized as follows: 
\setlength{\tabcolsep}{6pt}
\begin{center}
\begin{tabular}[b]{ p{.46\columnwidth} | p{.46\columnwidth} } %|
\textsf{Composition Controller} & \textsf{Supervisor} \\ 
\hline	
given a system history $h$ and a target action $a$, it
outputs one delegation $\Comp(h,a)$
& 
given a plant's string prefix $w$, it outputs enabled events $\sup(w)$ \\
\hline
such that it is possible to proceed forever
&
such that we can always `reach' marked states\\
\end{tabular}
\end{center}

Hence, the idea is to mimic $j=\Comp(h,a)$ by means of $j\in \sup(h\concat a)$, with $a\in \sup(h)$. 
However, there are fundamental differences between the two formalisms that do not allow for a direct, straightforward, translation.
As a matter of fact, a naive translation that defines the plant as the cross-product of all available behaviors and the target's language as specification will simply not work, due to several mismatches between DES and behavior composition (see 
Section~\ref{sec:conclusions} for details). In particular, it is well known that, for nondeterministic systems (as is the case with the available system), the notion of language inclusion is weaker than that of simulation.

From now on, let $\S=\tup{\B_1,\ldots,\B_n}$ be an available system,
where $\B_i=\tup{B_i,\BAct,b_{0i},\delta_i}$ for $i \in
\set{1,\ldots,n}$, and $\Targ = \tup{T,\BAct_t,t_0,\delta_t}$ a target
behavior (without loss of generality we assume $\T$ to be connected and all $B_i$'s and $\T$ to be mutually disjoint sets).
The general approach is to build an adequate plant from $\S$ and $\T$,
and define a specification language $K$, such that controlling $K$ (as
per Definition~\ref{def:controllable}) amounts to composing 
%module $\T$ in system $\S$.
$\T$ in $\S$.

So, let us next build a generator $\Plant$---the \emph{plant} to be controlled---from target $\T$ and system $\S$.
The controllable aspect of the plant amounts to behavior delegations: at any point in time, a supervisor can enable or disable an available behavior to execute.
On the other hand, the supervisor can control neither the action requests nor the evolution of the behavior selected---they are uncontrollable events.
A state in the plant encodes a snapshot of the whole composition process, namely, the state of all behaviors (including the target) together with the current pending target request and current behavior delegation.
Only those with no pending request or delegation are considered ``marked.'' 
Below, we use two auxiliary sets $\Services=\set{1,\ldots,n}$ and $\Nondet = \bigcup_{i \in \set{1,\ldots,n}} B_i$. \label{def:succ}

\begin{definition}
Let the \defterm{composition plant} $\Plant = \tup{\Act,G,g_0,\gamma,G_m}$ be defined as follows:
\begin{itemize}
\item $\Act = \Act_c \cup \Act_u$, where $\Act_c = \Services$ and  $\Act_u=  \BAct_t \cup \Nondet$, is the finite set of controllable (behaviors' indexes) and uncontrollable events (target's actions and behaviors' states).

\item $G = T \times B_1 \times \ldots \times B_n \times (\BAct_t \cup \set{e}) \times (\Services\cup\set{0})$ is the finite set of states of the plant. Additional symbol $e$ denotes no active request, whereas index $0$ denotes no active delegation.

\item $g_0 = \tup{t_0, b_{01},\ldots,b_{0n},e,0}$ is the initial state of the plant, encoding the initial configuration of the system and target, and the fact that there has been no request event or delegation.

\item $\gamma : G \times \Act \ra G$ of $\Plant$ is the plant's transition function where  
$\gamma(\tup{t,b_1,\ldots,b_n,\req,\idx},\sigma)$ is equal to:
\begin{itemize}
  \item $\tup{\delta_t(t,\sigma),b_1,\ldots,b_n,\sigma,0}$ if $\req = e$, $\idx = 0$,  $\sigma \in \BAct_t $; 
  
  \item $\tup{t,b_1,\ldots,b_n,\req,\sigma}$ if $\req \in\BAct_t$, $\idx =0$, $\sigma \in \Services$;

\item $\tup{t,b_1,\ldots,b_{\idx}',\ldots,b_n,e,0}$ if $\sigma \!\! \in\!\! \delta_{\idx}(b_{\idx},\req)$, $b_{\idx}'\!=\!\!\sigma$.

\end{itemize} 
% %%
\item $G_m = T \times B_1 \times \ldots \times B_n \times \set{e}\times \set{0}$ (marked states).
\end{itemize}
\end{definition}

By inspecting the plant transition function $\gamma$ we can see that the whole process for one target request involves three transitions in the plant, namely, target action request, behavior delegation, and lastly available system evolution.
Initially, and after each target request has been fulfilled, the plant is in a state with no active request ($e$) and no behavior delegation ($0$), ready to accept and process a new target request---a \emph{marked} state. Then:
\begin{enumerate}
  \item  given a legal target request (uncontrollable event) $\sigma \in \BAct_t$, the plant evolves to a state recording the request and the corresponding target evolution (case 1 of $\gamma$);
  
  \item after that, the plant may evolve relative due to (controllable) delegation events (one per available behavior), to states recording such delegations as well as the current pending action (case 2 of $\gamma$); and finally
  
  \item the plant may evolve, in an uncontrollable manner, to states reflecting all possible evolutions of the behavior selected,  together with no active request or delegation (case 3 of $\gamma$).   
\end{enumerate}
 
Observe that a composition plant $\Plant$, being a generator, is \emph{deterministic}, whereas the available behaviors being modelled may include nondeterministic evolutions. The fact is that such nondeterminism is encoded via uncontrollable events. 

\begin{figure*}[t!]
\resizebox{\textwidth}{!}{
\setlength{\tabcolsep}{3pt}
\begin{tabular}[b]{  c  c  c  c | c c c c | c c c c  | c c c c  }
 $g_0$ & $\tup{t_0,a_0,b_0,c_0}$ & $e$          &  0 & $g_8$ & $\tup{t_1,a_0,b_0,c_0}$ & $\amine$ &  1  & $g_{16}$ &$\tup{t_2,a_3,b_0,c_1}$ & $e$ & 0 & $g_{24}$ & $\tup{t_3,a_1,b_0,c_0}$ & $\adepo$ &  3 \\ \hline
$g_1$ & $\tup{t_0,a_0,b_0,c_0}$ & $\adrill$ &  0  & $g_9$ & $\tup{t_1,a_3,b_0,c_0}$ & $e$ &  0 & $g_{17}$ & $\tup{t_3,a_3,b_0,c_0}$ & $\adepo$ &  0 & $g_{25}$ & $\tup{t_3,a_3,b_0,c_0}$ & $\adepo$ &  1\\ \hline
$g_2$ & $\tup{t_0,a_0,b_0,c_0}$ & $\adrill$ &  1  & $g_{10}$ & $\tup{t_1,a_1,b_0,c_0}$ & $e$ &  0 & $g_{18}$ & $\tup{t_2,a_1,b_0,c_0}$ & $\aload$ & 0 & $g_{26}$ & $\tup{t_3,a_3,b_0,c_0}$ & $\adepo$ &  2 \\ \hline
$g_3$ & $\tup{t_0,a_0,b_0,c_0}$ & $\adrill$ &  2  & $g_{11}$ & $\tup{t_2,a_3,b_0,c_0}$ & $\aload$ &  0 & $g_{19}$ & $\tup{t_2,a_1,b_0,c_0}$ & $\aload$ &  1& $g_{27}$ & $\tup{t_3,a_1,b_1,c_0}$ & $e$ &  0\\ \hline
$g_4$ & $\tup{t_0,a_0,b_0,c_0}$ & $\adrill$ &  3  & $g_{12}$ & $\tup{t_2,a_3,b_0,c_0}$ & $\aload$ &  1 & $g_{20}$ & $\tup{t_2,a_1,b_0,c_0}$ & $\aload$ &  2& $g_{28}$ & $\tup{t_3,a_3,b_0,c_1}$ & $\adepo$ &  0   \\ \hline
$g_5$ & $\tup{t_1,a_0,b_0,c_0}$ & $\amine$ &  0 & $g_{13}$ & $\tup{t_2,a_3,b_0,c_0}$ & $\aload$ &  2 & $g_{21}$ & $\tup{t_2,a_1,b_0,c_0}$ & $\aload$ &  3& $g_{29}$ & $\tup{t_3,a_3,b_0,c_1}$ & $\adepo$ &  2 \\ \hline
$g_6$ & $\tup{t_1,a_0,b_0,c_0}$ & $\amine$ &  3 & $g_{14}$ &$\tup{t_2,a_3,b_0,c_0}$ & $\aload$ &  3 & $g_{22}$ & $\tup{t_2,a_1,b_0,c_0}$ & $e$ &  0  & $g_{30}$ & $\tup{t_3,a_3,b_0,c_1}$ & $e$ &  0 \\ \hline
$g_7$ & $\tup{t_1,a_0,b_0,c_0}$ & $\amine$ &  2  & $g_{15}$ & $\tup{t_2,a_3,b_0,c_1}$ & $e$ &  0 & $g_{23}$ & $\tup{t_2,a_1,b_0,c_0}$ & $e$ &  0 & $g_{31}$ & $\tup{t_4,a_3,b_1,c_0}$ & $\aunload$ &  0\\ \hline
 \end{tabular}
}
\end{figure*}
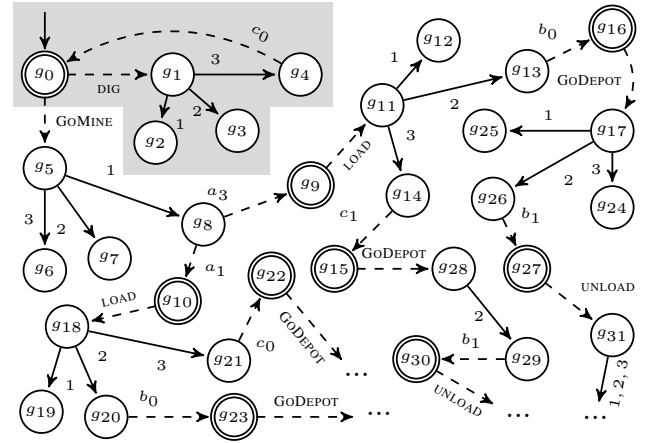
\begin{figure}[h!]
% \vspace{-.2in}
\resizebox{\columnwidth}{!}{\begin{tikzpicture}[node distance=1.5cm,auto]
\tikzset{ActionStyle/.style = {font=\tiny}}
\tikzstyle{every state}=[fill=white,font=\tiny,circle,draw=black,text=black,inner sep=1.pt,minimum size=5mm]

\begin{scope}[shift={(0cm,0cm)}]
\node[state,initial above,double]	(q0)                    {$g_0$};
\node[state]   (q1) [right of=q0]   {$g_1$};
\node[state]   (q2) [below of=q1,yshift=.7cm,xshift=-.2cm]   {$g_2$};
\node[state]   (q3) [below right of=q1,yshift=.4cm,xshift=-.3cm]   {$g_3$};
\node[state]   (q4) [right of=q1]   {$g_4$};
\node[state]   (q5) [below of=q0,yshift=.4cm]   {$g_5$};
\node[state]   (q6) [below of=q5,yshift=.3cm]   {$g_6$};
\node[state]   (q7) [below right of=q5,xshift=-.3cm]   {$g_7$};
\node[state]   (q8) [right of=q7,shift={(-.4cm,.4cm)}]   {$g_8$};
\node[state,double]   (q9) [above right of=q8,shift={(.2cm,-.6cm)}]   {$g_9$};
\node[state,double]   (q10) [below of=q8,shift={(-.3cm,.6cm)}]   {$g_{10}$};
\node[state]   (q11) [right of=q3,shift={(.2cm,.3cm)}]   {$g_{11}$};
\node[state]   (q12) [above right of=q11,shift={(-.4cm,-.3cm)}]   {$g_{12}$};
\node[state]   (q13) [right of=q11,shift={(.2cm,.4cm)}]   {$g_{13}$};
\node[state]   (q14) [below right of=q11,xshift=-.3in]   {$g_{14}$};
\node[state,double]   (q15) [below left of=q14,shift={(.2cm,.2cm)}]   {$g_{15}$};
\node[state,double]   (q16) [right of=q13,xshift=-.5cm,yshift=.5cm]   {$g_{16}$};
\node[state]   (q17) [below of=q16,yshift=.3cm]   {$g_{17}$};
\node[state]   (q24) [below of=q17,shift={(0cm,.6cm)}]   {$g_{24}$};
\node[state]   (q25) [left of=q17,shift={(-.0cm,0cm)}]   {$g_{25}$};
\node[state]   (q26) [below of=q25,shift={(.1cm,.7cm)}]   {$g_{26}$};
\node[state,double]   (q27) [below of=q17,shift={(-1cm,-.12cm)}]   {$g_{27}$};
\node[state]   (q18) [left of=q10,shift={(.2cm,-.3cm)}]   {$g_{18}$};
\node[state]   (q19) [below of=q18,shift={(-.3cm,.5cm)}]   {$g_{19}$};
\node[state]   (q20) [below right of=q18,xshift=-.6cm]   {$g_{20}$};
\node[state]   (q21) [below of=q8,shift={(.3cm,-.1cm)}]   {$g_{21}$};
\node[state,double]   (q22) [right of=q21,shift={(-1cm,1cm)}]   {$g_{22}$};
\node[state,double]   (q23) [right of=q20,shift={(0cm,0cm)}]   {$g_{23}$};
\node[state]   (q28) [right of=q15,shift={(-.1cm,0cm)}]   {$g_{28}$};
\node[state]   (q29) [below right of=q28,shift={(-.2cm,.0cm)}]   {$g_{29}$};
\node[state,double]   (q30) [left of=q29,shift={(.2cm,.0cm)}]   {$g_{30}$};
\node[state]   (q31) [below of=q24,shift={(.0cm,.0cm)}]   {$g_{31}$};

\node[] (c1) [below of=q29,shift={(-.1cm,.8cm)}]   {$\cdots$};
\node[] (c2) [below of=q22,shift={(1cm,.3cm)}] {$\cdots$};
\node[] (c3) [right of=q23,shift={(.2cm,0cm)}] {$\cdots$};
\node[] (c4) [below of=q31,shift={(-.2cm,.5cm)}] {$\cdots$};

\node[] (dummy) [above of=q2,yshift=-.1cm] {};

\path	
(q0) edge[dashed]  node[swap,ActionStyle] {$\adrill$}  (q1)
(q1) edge[]  node[ActionStyle] {$1$}  (q2)
(q1) edge[below]  node[near start,ActionStyle] {$2$}  (q3)
(q1) edge[]  node[ActionStyle,near start] {$3$}  (q4)
(q4) edge[bend right,dashed]  node[sloped,very near start,above,ActionStyle] {$c_0$}  (q0)
(q0) edge[dashed]  node[ActionStyle] {$\amine$}  (q5)
(q5) edge[]  node[swap,ActionStyle] {$3$}  (q6)
(q5) edge[]  node[swap,ActionStyle] {$2$}  (q7)
(q5) edge[]  node[near start,ActionStyle] {$1$}  (q8)
(q8) edge[dashed]  node[near start,ActionStyle] {$a_3$}  (q9)
(q8) edge[dashed]  node[near start,ActionStyle] {$a_1$}  (q10)
(q9) edge[dashed]  node[sloped,below,ActionStyle] {$\aload$}  (q11)
(q11) edge[]  node[ActionStyle] {$1$}  (q12)
(q11) edge[below]  node[ActionStyle] {$2$}  (q13)
(q11) edge[]  node[ActionStyle] {$3$}  (q14)
(q14) edge[dashed]  node[swap,ActionStyle] {$c_1$}  (q15)
(q13) edge[dashed]  node[ActionStyle] {$b_0$}  (q16)
(q16) edge[dashed,bend left]  node[swap,ActionStyle] {$\adepo$}  (q17)
(q17) edge[]  node[swap,ActionStyle] {$1$}  (q25)
(q17) edge[]  node[swap,ActionStyle] {$3$}  (q24)
(q17) edge[]  node[ActionStyle] {$2$}  (q26)
(q26) edge[dashed]  node[ActionStyle] {$b_1$}  (q27)
(q10) edge[dashed]  node[sloped,above,ActionStyle] {$\aload$}  (q18)
(q18) edge[]  node[ActionStyle] {$1$}  (q19)
(q18) edge[]  node[ActionStyle] {$2$}  (q20)
(q18) edge[]  node[swap,near end,ActionStyle] {$3$}  (q21)
(q21) edge[dashed]  node[swap,near start,ActionStyle] {$c_0$}  (q22)
(q20) edge[dashed]  node[near start,ActionStyle] {$b_0$}  (q23)
(q15) edge[dashed]  node[ActionStyle] {$\adepo$}  (q28)
(q28) edge[]  node[near start,below,ActionStyle] {$2$}  (q29)
(q29) edge[dashed]  node[above,ActionStyle] {$b_1$}  (q30)
(q27) edge[dashed]  node[ActionStyle] {$\aunload$}  (q31)

(q22) edge[dashed]  node[sloped,below,ActionStyle] {$\adepo$}  (c2)
(q23) edge[dashed]  node[sloped,above,ActionStyle] {$\adepo$}  (c3)
(q30) edge[dashed]  node[sloped,below,ActionStyle] {$\aunload$}  (c1) 
 (q31) edge[]  node[sloped,below,ActionStyle] {$1,2,3$}  (c4)

;

\begin{pgfonlayer}{background}
\node [fill=black!15,fit=(q0) (q1) (q4) (dummy)] {};
\node [fill=black!15,fit=(q2) (q3)] {};
\end{pgfonlayer}

\end{scope}

\end{tikzpicture}}
\caption{Plant $\Plant$ for the example in Figure~\ref{fig:ambientIntel-system} (partial). Double circled states are marked, so any word prefix ending in one of these states is marked. Dashed transitions correspond to uncontrollable events, solid ones to controllable events (delegations). State components are listed in the table.}
\label{fig:plant}
% \vspace{-.22in}
\end{figure}

Hence, the final step in an action request delegation process may yield multiple plant states, one per nondeterministic evolution of the selected behavior. Also, such resulting states are to be considered ``marked,'' in that a complete delegation process has been completed.
If, however, the chosen behavior is unable to legally execute the active request from its current state, then no transition is defined and the plant (non-marked) state is a dead-end.  

\begin{example}
Figure~\ref{fig:plant} depicts the (partial) plant for the composition problem of Figure~\ref{fig:ambientIntel-system}. 
Each complete delegation process of action requests corresponds, in the plant, to three consecutive events in $(\BAct_t \concat \Services
\concat \Nondet)$.

After each uncontrollable event representing a target request, three delegations---to available behaviors $\B_1,\B_2$ and $\B_3$---are \emph{always} possible.
For instance, the nodes in the greyed area represent the complete delegation of the digging action from the initial plant (and composition) state $g_0$. The event $\adrill$ represents the action request; that is uncontrollable, and hence always enabled. The resulting state $g_1$ registers such request. Then, three distinct controllable events embody the three possible delegations, one per  available behavior. However, only behavior $\B_3$ can legally perform action $\adrill$ from its initial state (see Figure~\ref{fig:ambientIntel-system}) to successor state $c_0$. Hence, a further uncontrollable event ($c_0$ itself) is used to model the looping transition evolution of $\B_3$. 
In general, there could be multiple uncontrollable evolutions if the delegated behavior behaves nondeterministically; see for example, plant state $g_8$ where behaviour $\B_1$ may evolve in two ways.

Of course, delegations reaching dead-end states are not desirable (e.g., delegation $1$ in $g_{18}$). 
However, not reaching an immediate dead-end is not enough to capture the composition requirements. Indeed, whereas delegation to $\B_2$ and $\B_3$ will avoid immediate dead-ends in $g_{18}$, only the latter will be part of a composition solution (see later 
Figure~\ref{ex:CG}, state $2$). 
\end{example}

With the plant built, the question is what language one would like to control. 
The answer is simple: we aim to control exactly the marked language of the composition plant, that is, 
\[
	\Specc = \Mkd(\Plant).
\]
In other words, we seek for ways of always controlling the plant so as to eventually be able to reach the end of every request-delegation process. 
Observe that, contrary to intuition, the target behavior $\Targ$ is \emph{not} used to derive the language specification, except in that it is embedded into the plant itself. 
This is not surprising, as the the target is one of the components generating uncontrollable events (the other being the evolution of available behaviors).
 
We shall claim that the ability to control $\Specc$ in plant $\Plant$ amounts to the ability to compose $\T$ in system $\S$. 
To that end, we first show an important technical result stating that
set of $(P,\tau)$-induced system histories $\H_{\Comp,\tau}$ is in
bijection with the set of traces in $\TSpecc$ when $P$ is a composition controller and $\tau$ a trace of $\Targ$. 
This appears evident when carefully inspecting Figure~\ref{fig:plant}, and is formalized in the following lemma.
We use mapping $\word(h) \in(\BAct_t \concat \Services \concat \Nondet)^{\card{h}}$ to translate a system history (i.e., a finite trace of the enacted system $\B_\S$) into words generated by composition plant $\Plant$.

\begin{lemma}\label{lemma:final}
Controller $\Comp$ is a composition for target $\Targ$ in system $\S$ \emph{iff} for each target trace $\tau$ and system history $h\in \H_{\Comp,\tau}$ we have that $\word(h)\in\TSpecc$, where:
\[
\begin{array}{l}
	\word(\vec{b}_0\goto{a_1,j_1}\vec{b}_1\goto{a_2,j_2} \ldots \goto{a_\ell,j_\ell} \vec{b}_\ell) = \\
	( a_1 \concat j_1 \concat  \stateproj_{j_1}(\vec{b}_1)) \concat
 			\; \ldots \; \concat ( a_\ell \concat j_\ell \concat \stateproj_{j_\ell}(\vec{b}_\ell) ).
\end{array}
\]
Functions $\stateproj_i : G \ra B_i$\footnote{We extend the function $\stateproj_i$ to map a plant state to  corresponding behavior state.} , with $i\in\Services$, project the state of $i$-th behavior in a plant state, that is, 
$\stateproj_i(\tup{t,\vec{b},a,j})=st_i(\vec{b})=b_i$. Analogously, for the target, $\stateproj_t(\tup{t,\vec{b},a,j}) = t$.
\end{lemma}

\begin{proof}
$(\Rightarrow)$ Assume by contradiction that there exists a composition $P$ such that for some target trace $\tau$ and induced history $h\goto{a,j}\vec{b}$, we have $P(h,a)=j$ but $\word(h) \concat a \concat  j \concat b' \not\in \TSpecc$, with $b'= \stateproj_j(\vec{b})$.
This implies that $\word(h) \concat a \concat j \concat b' $ is not allowed from the initial state $g_0$ of the plant, according to supervisor $\sup$ such that $\Mkd(\sup/\G) =  \TSpec$, i.e., either 
\begin{itemize}\itemsep=0pt
\item[$(1)$] $\word(h) \concat a \not\in
\Gtd(\Plant)$ or 
\item[$(2)$] $\word(h) \concat a \concat j \not\in
\Gtd(\Plant)$ or
\item[$(3)$] $\word(h) \concat a \concat j \concat b' \not\in
\Gtd(\Plant)$ or
\item[$(4)$] for all words $w\in\Mkd(\Plant)$ with $w>\word(h) \concat a \concat
j \concat  b'$ we have $w\not\in\Closure{\TSpecc}$. 
\end{itemize}
Case $(1)$ is not possible by construction of $\Plant$. Indeed,
according to $\gamma$, it is 
$w\concat a\in\Gtd(\Plant)$ for every
$w\in\Mkd(\Plant)$ such that $\delta_t(\stateproj_t(\gamma(g_0,w)),a)$ is defined in
$\Targ$.
Case $(3)$ implies, by definition of $\gamma$, that $b'_j \not\in \delta_j(b_j,a)$, with $b_j = \stateproj_j(\gamma(g_0,w))$. Hence, the action $a$ can
not be replicated by behavior $\Be{j}$ and, as a consequence, the
plant's state reached with $a$ is a dead-end. 
This contradicts the fact that $P$ is a composition
for $\Targ$ in $\S$. 
Indeed, let
$g=\gamma(g_0,w)$; note that this also implies
that
$\stateproj_t(g)\not\preceqnd\tup{\stateproj_1(g),\ldots,\stateproj_n(g)}$,
namely that the simulation is violated, as action $a$ can not be
replicated in the enacted system state $\tup{\stateproj_1(g),\ldots,\stateproj_n(g)}$.  
By following the same argument, we can also exclude case $(2)$. 
Finally, case $(4)$ implies that for any such word $w$ we have $w\concat\Act_u\cap\Gtd(\Plant)\not\subseteq\Closure{\TSpecc}$, i.e., there
exists a sequence of (uncontrollable) events leading to a 
state which is not coreachable, i.e., from where a marked state is not
reachable. Indeed, remember that $\Specc =
\Mkd(\Plant)$. Hence, since every action $a\in\BAct_t\subset\Act_u$ is always allowed by any supervisor and,
by construction of $\Plant$, $w\concat a\in\Gtd(\Plant)$ for every
$w\in\Mkd(\Plant)$, we can apply the same reasoning of $(3)$ and deduce that $\Comp$ is
not a composition. That is, there exists a target
trace $\tau'=\tau\goto{a}t_\ell$, with $h\in\H_{\Comp,\tau}$ and $w=\word(h)$, not realized by $\Comp$.

$(\Leftarrow)$ 
First of all, since $\Closure{\Specc}\cap\Gtd(\Plant)\subseteq \Specc$ and by the
previous assumption $\TSpecc\neq\emptyset$, then by
Theorem~\ref{theo:supNSC} a supervisor $\sup$ does exist. 
Hence, $\word(h) \concat a\concat j \concat b' \in \TSpecc$ iff there
exists a supervisor $\sup$ such that $\Mkd(\sup/\Plant)=\TSpecc$,
$a\in \sup(\word(h))$, $j\in \sup(\word(h) \concat a)$ and
$b'\in \sup(\word(h) \concat a \concat j)$. 
Then, remember that
$\BAct_t\subset\Act_u$ and hence all target action are always allowed
by $\sup$. Similarly, the event set $\Nondet$ is uncontrollable as well.
Assume by contradiction that $\word(h) \concat a\concat j \concat b' \in
\TSpecc$ but it does not
exist any composition $\Comp$ such that $\Comp(h,a)=j$. By
definition of composition, this implies that there exists a target
trace $\tau=t_0\goto{a_1}\cdots t_k$ with $a=a_{\card{\tau}}$ such
that for some history $h\in\H_{\Comp,\tau}$ we have that
$\delta_\S(last(h),a,j)$ is not defined in the system behavior
$\B_\S$ built out of $\Behaviors$. 
This means that either $\delta_t(\stateproj_t(last(h)),a)$ is not
defined or behavior $\Be{j}$ can not perform this action from its
current state $\stateproj_j(last(h))$, i.e., $b'\neq
\delta_j(\stateproj_j(last(h)),a)$. Again, observe how
this also implies that
$\stateproj_t(last(h))\not\preceq\tup{\stateproj_1(last(h)),\ldots,\stateproj_n(last(h))}$. 
In other words, according to
$\gamma$, $\word(h)\concat a\concat
j\concat b' \not\in\Gtd(\Plant)$ . If this is the case, then either $a\not\in \sup(\word(h))$
or $j\not\in \sup(\word(h) \concat a)$ or $b'\not\in \sup(\word(h)
\concat a \concat j)$ and we get a contradiction. 
\end{proof}

The above lemma is the key to prove our main results of this section, namely, that
supervisors able to control the specification $\Specc$ in plant $\Plant$ correspond \emph{one-to-one} with composition (solution) controllers for building target $\T$ in available system $\S$.
To express such results, we first need to relate supervisors and controllers.  

\begin{definition}
Let $\sup$ be a supervisor for composition plant $\Plant$.
A controller $\Comp_\sup : \H \times \BAct \ra \set{1,\ldots,n}$ is \defterm{induced by} $\sup$ \emph{iff} $\Comp_\sup(h,\sigma) \in
\sup(\word(h)\concat \sigma)$, for every $h \in \H$ and $\sigma \in \BAct$.
\end{definition}

In other words, a 
%controller 
$\Comp_\sup$ is induced by a supervisor
$\sup$ iff its delegations fall into the set of ``delegation events'' allowed by $\sup$. Clearly, a supervisor can induce many controllers. 

The main result of this section states that the supremal of the specification is controllable by some supervisor iff a solution to the composition problem exists. Moreover, every such supervisor induces controllers that are in fact compositions.% solutions. 

\begin{theorem}[Soundness]\label{theo:soundness}
There exists a nonblocking supervisor $\sup$ such that $\Mkd(\sup/\Plant)=\TSpecc\neq\emptyset$ \emph{iff} there exists a composition $\Comp$ for $\Targ$ in $\S$. 
In particular, every controller $P_V$ induced by $\sup$ is a composition for $\T$ in $\S$.
\begin{proof}
$(\Rightarrow)$
Assume by contradiction that for some
controller $\Comp_\sup$ there exists a  target trace
$\tau=t_0\goto{a_1}\cdots \goto{a_\ell}t_\ell$ and an 
induced system history $h\in\H_{\Comp_\sup,\tau}$ such that either
$(1)$ $\Comp_\sup(h,a_\ell)$ is not defined or 
$(2)$ $\Comp_\sup(h,a_\ell)=j$ but 
%behavior $\Be{j}$ can not perform such an action, i.e., 
$\delta_j(\stateproj_j(\last(h)),a_\ell)$ is not defined.
By Lemma~\ref{lemma:final}, it means that $\word(h)\concat a_\ell \concat
j\not\in\Closure{\TSpecc}$. More precisely, $(1)$ implies that $\word(h) \concat a_\ell \not\in
\Gtd(\Plant)$ whereas $(2)$ implies that $a_\ell\not\in
\sup(\word(h))$ and $j\not\in \sup(\word(h) \concat
a_\ell)$.
%%
%Also, by construction of $\Plant$, we have that $\TSpecc \not= \emptyset$.
%%
Hence, either $\TSpecc = \emptyset$, or $\Comp_\sup$ does not realize
the target trace $\tau$ and we contradict Lemma~\ref{lemma:final}. 

$(\Leftarrow)$ By Lemma~\ref{lemma:final}, if such $\Comp$ exists then
$\TSpecc\neq\emptyset$.
\end{proof}
\end{theorem}

\noindent
Furthermore, (nonblocking) supervisors are ``complete'' in that they embed every possible composition controller.

\begin{theorem}[Completeness]\label{theo:completness}
Given a nonblocking supervisor $\sup$ such that $\Mkd(\sup/\Plant)=\TSpecc$, every composition $\Comp$ for $\Targ$ in $\S$ is induced by $\sup$.
\end{theorem}
\begin{proof}
Assume by contradiction that there exists a composition $\Comp'$ which
can not be induced by $\sup$, that is, it is such that
$\Comp'(h,a_\ell) \not\in \sup(\word(h)\concat a_\ell)$ for some
target trace $\tau=t_0\goto{a_1}\cdots \goto{a_\ell} t_\ell$ and some
induced system history $h\in\H_{\Comp',\tau}$. More precisely, for
some $j\in\Services$, $\word(h)\cdot a_\ell \cdot j\not\in \TSpecc$
but $j\in \Comp'(h,a_\ell)$.
It is easy to see that, by Lemma~\ref{lemma:final}, $\Comp'$ can not
be a composition.
\end{proof}

Let us call supervisors of this sort \defterm{composition supervisors}.
These two results demonstrate the formal link between the two synthesis tasks, namely, synthesis of a composition controller and supervisor synthesis.

Recalling that $\Specc = \Mkd(\Plant)$ and the definition of
compositions, and as already hinted in the proof of Lemma~\ref{lemma:final}, we can also explicitly relate the notions of maximal controllable
sublanguage and \textsc{nd}-simulation.

\begin{corollary}%\todo{look at this}
\label{corolllary:1} 
Given a plant $\Plant$ for $\S$ and $\T$ as above, if $\TSpecc\neq\emptyset$ then for any word $w\in\TSpecc$, we have that
$\stateproj_t(g)\preceqnd\tup{\stateproj_1(g),\ldots,\stateproj_n(g)}$
where $g=\gamma(g_0,w)$.
\end{corollary} 

\begin{proof}
We proceed by induction on the length of $w$. The claim holds for
$g=g_0$ (by Theorem~\ref{theo:soundness} and Thorem~\ref{th:sim_initial_states}).  
Assume now it holds for $w'$, with $g'=\gamma(g_0,w')$, and consider any
word 
%$w>w'$. By the definition of $\Plant$ and its marked
%language (recall that $\TSpecc=\Mkd(\Plant)$), it must be 
$w=w'\concat
a \concat j\concat b_j$ for some
action $a$, delegation $j$ and behavior state $b_j$. Since $a\in\Sigma_u$, and because
$\TSpecc$ is controllable, from Definition~\ref{def:controllable} it
follows that $w\in\TSpecc$ implies $w\concat a\in
\Closure{\TSpecc}$, and for the same reason
also $w\concat a\concat j \concat b_j\in
\TSpecc$ for some $j\in\Services$ and $b_j\in B_j$. Which means that for any
action $a$ such that $\stateproj_t(g')\goto{a} t$ for
some target state $t$, there exists an index $j$ such that
$\tup{\stateproj_1(g),\ldots,\stateproj_n(g)}\goto{a,j}\vec{b}'$ in
$\B_\S$, with $\stateproj_j(\vec{b})=b_j$.
\end{proof}

It remains to be seen, though, how to actually \emph{extract} finite representations of composition controllers from DES composition supervisors.

\begin{figure*}[t!]
\resizebox{\textwidth}{!}{
\begin{tabular}[b]{  c c | c c | c c | c c | c c | c c  }% |
    \hline			
    0 & $\tup{t_0,a_0,b_0,c_0}$ & 4 & $\tup{t_1,a_3,b_0,c_0}$ & 8 & $\tup{t_4,a_3,b_0,c_0}$ & 12 & $\tup{t_4,a_3,b_2,c_1}$ & 16 &  $\tup{t_1,a_1,b_0,c_1}$ &  20 & $\tup{t_2,a_1,b_0,c_1}$ \\ \hline
    1 & $\tup{t_1,a_3,b_0,c_0}$ & 5 & $\tup{t_2,a_1,b_0,c_1}$ & 9 & $\tup{t_4,a_0,b_0,c_1}$ & 13 &  $\tup{t_0,a_0,b_0,c_1}$ & 17 & $\tup{t_1,a_3,b_0,c_1}$ & 21 & $\tup{t_2,a_1,b_0,c_1}$\\ \hline
    2 & $\tup{t_1,a_1,b_0,c_0}$ & 6 &  $\tup{t_3,a_2,b_0,c_1}$ & 10 & $\tup{t_4,a_3,b_2,c_0}$ & 14 & $\tup{t_0,a_3,b_3,c_1}$ & 18 & $\tup{t_2,a_3,b_0,c_1}$ & &\\ \hline
    3 & $\tup{t_2,a_3,b_0,c_0}$ & 7&  $\tup{t_0,a_3,b_3,c_0}$& 11 & $\tup{t_3,a_3,b_1,c_0}$ & 15 & $\tup{t_3,a_3,b_1,c_1}$ & 19 & $\tup{t_2,a_3,b_0,c_1}$ &&\\ 
\hline
\end{tabular}
}
\end{figure*} 

\subsection{From Supervisors to Controller Generators}\
As discussed at the end of Section \ref{sec:preliminaries_bc}, a controller generator (CG) is a finite structure encoding all possible composition solutions---a sort of a universal solution----that, once computed, can be used at runtime to produce all possible target realizations.
Because of that, CGs have been shown to enjoy run-time flexibility and robustness properties, in that the executor can leverage on them to recover or adapt to behavior various types of execution failures (e.g., an available behavior breaking down completely)~\cite{DeGiacomoPatriziSardina:AIJ13}.
%%
%In this section, 
Next, we show that it is possible to to extract \emph{the} CG from composition supervisors. 
%However, our approach does not rely on the notion of simulation as in \cite{DeGiacomoPatriziSardina:AIJ13}.

We start by noting that, since both languages $\Gtd(\Plant)$ and $\Specc$ are regular, they are implementable.
In fact~\citea{WonhamRamadge:SIAMJCO87} have shown that it is possible to compute a generator $\R$ that represents exactly the behavior of controlled system $\sup/\Plant$, for some supervisor $V$ able to control $\TSpecc$. 
Such generator $\R$ will capture not only the control actions of supervisor $V$, but also all internal (uncontrollable) events of the plant.
In a nutshell, extracting the controller generator amounts to projecting out the latter and transforming controllable events into behavior delegations. The whole procedure can be depicted as:
%%
% \begin{figure}[h!]\label{fig:procedure}
\vspace*{-.2in}
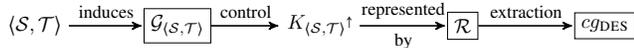
\begin{center}
\resizebox{\columnwidth}{!}{\begin{tikzpicture}[->,thin,node distance=2.5cm]
\tikzstyle{every state}=[circle,fill=none,draw=black,text=black,inner
sep=4pt,minimum size=4mm,font=\small]

\tikzset{ActionStyle/.style = {font=\small}}

\node[]		(0)                   {$\tup{\S,\T}$};
\node[draw=black,rectangle,right of=0]  (1)    {$\Plant$};
\node[]    	(2) [right of=1]   {$\TSpecc$};
\node[draw=black,rectangle]    (3) [right of=2]   {$\R$};
\node[draw=black,rectangle]    (4) [right of=3]   {$\CGDES$};

\path	 (0) edge[thick]  node[ActionStyle] {induces}  (1);
\path	 (1) edge[thick]  node[ActionStyle] {control}  (2);
\path	 (2) edge[thick]  	node[ActionStyle,above] {represented}  
							node[ActionStyle,below] {by} (3);
\path	 (3) edge[thick]  node[ActionStyle] {extraction}  (4);

\end{tikzpicture}}
%\caption{}
\end{center}
\vspace{-.1in}
% \end{figure}

From a composition problem, a corresponding plant is first built. If the composition is solvable, the language $\TSpecc \not=\emptyset$ is controllable (Theorems~\ref{theo:soundness} and \ref{theo:completness}). 
In addition, we know that there exists a generator $\R$ representing a composition supervisor.
%%

%Given a state $y$ of $\R$, we denote with $[y]$ the tuple $\tup{\stateproj_t(y),\stateproj_1(y),\ldots,\stateproj_n(y)}$ extracting the the full composition state from $\R$'s state $y$, where function $\stateproj_i(y)$ projects the local state of target and that of the available behaviors (this can be deterministically reconstructed from the event labels in $\R$ in linear time).
%=======
Given a state $y$ of $\R$, we denote with $[y]$ the tuple $\tup{\stateproj_t(y),\stateproj_1(y),\ldots,\stateproj_n(y)}$ extracting the full composition state from $\R$'s state $y$, where function $\stateproj_i(y)$ projects the local state of target and that of the available behaviors (this can be deterministically reconstructed from the event labels in $\R$ in linear time on the size of $\R$).

\begin{definition}
Let $\R = \tup{ \Act, Y, y_0 , \rho, Y_m }$ be the generator
representing a 
%composition 
supervisor.
The \defterm{DES controller generator} is a finite-state structure  $\CGDES = \tup{A_t,\Services,Q,[y_0],\vartheta,\omega}$, where: 
\begin{itemize}%\itemsep=0pt
\item $A_t$ and $\Services$ are the set of target actions and behavior indexes, as before;

\item $Q = \set{ [y] \mid y\in Y,\; y=\rho(y_0,p),\; p\in(\BAct_t \concat \Services\concat \Nondet)^*}$ is the set of full composition states reachable from initial generator's state $y_0$; \footnote{Recall that \Nondet \ is the set of all behaviors states as defined on page~\pageref{def:succ}.} 

\item $[y_0]$ is the initial state of $\CGDES$;

\item $\vartheta : Q \times \BAct_t\times \Services \times Q$ is the transition relation such that $[y']\in\vartheta([y],\sigma,j)$ \emph{iff} $y' = \rho(y,\sigma \concat j \concat b'_j)$ for some $b_j' \in \Nondet$. 
That is, $\vartheta$ outputs a transition corresponding to the delegation of action $\sigma$ to the $j$-th module iff there exists a transition, labeled with $\sigma$, from its current state (namely, iff there exists a $\sigma$-successor $b'_j)$; and

\item $\omega : Q \times \BAct_t \ra 2^{\Services}$ is the behavior
  selection function, such that $\omega(q,\sigma) = \set{ j \mid
    \exists q' \in \vartheta(q,\sigma,j)}$, which just ``reads'' the function $\vartheta$.
\end{itemize}
\end{definition}

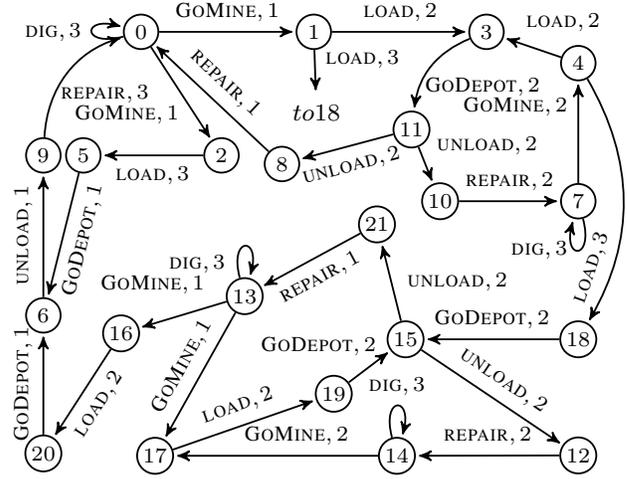
\begin{figure}
\centering
\resizebox{1\columnwidth}{!}{\begin{tikzpicture}[node distance=1.6cm,auto]
\tikzset{ActionStyle/.style = {font=\scriptsize}}
\tikzstyle{every state}=[fill=white,font=\scriptsize,circle,draw=black,text=black,inner sep=1pt,minimum size=4mm]

\begin{scope}[shift={(0cm,0cm)}]
\node[state]	(a0)  []                  {$0$};
\node[state]    	(a2) [right of=a0,shift={(0.4cm,0cm)}]   {$1$};
\node[state]    	(a4) [right of=a2,shift={(0.4cm,0cm)}]   {$3$};
\node[state]    	(a8) [below right of=a2] {$11$};
\node[state]    	(a12) [below right of=a8,shift={(-0.8cm,.3cm)}] {$10$};
\node[state]    	(a10) [left of=a8,shift={(.1cm,-.4cm)}] {$8$};
\node[state]    	(a3) [below right of=a0,shift={(-.2cm,-.3cm)}] {$2$};
\node[state]    	(a11) [below left of=a0,yshift=-.3cm] {$9$};
\node[state]    	(a19) [right of=a12] {$7$};
\node[state]    	(a23) [above of=a19] {$4$};
\node[state]    	(a9) [below of=a12,shift={(-0.4cm,0cm)}] {$15$};
\node[state]    	(a24) [below left of=a9,shift={(0.3cm,0.5cm)}] {$19$};
\node[state]    	(a6) [left of=a3,shift={(0cm,0cm)}] {$5$};
\node[state]    	(a7) [below of=a11,yshift=-.1in] {$6$};
\node[state]    	(a13) [below of=a8,shift={(-0.4cm,0.5cm)}] {$21$};
\node[state]    	(a16) [below left of=a13,shift={(-0.4cm,0.3cm)}] {$13$};
\node[state]    	(a5) [below of=a19] {$18$};
\node[state]    	(a14) [below of=a5,yshift=.1in] {$12$};
\node[state]    	(a18) [below of=a9,shift={(-0.1cm,.25cm)}] {$14$};
\node[state]    	(a20) [below left of=a16,shift={(-.3cm,.7cm)}] {$16$};
\node[state]    	(a25) [below of=a7,yshift=0in] {$20$};
\node[state]    	(a21) [left of=a18,shift={(-1.2cm,0cm)}] {$17$};

\node[] (c1) [below right of=a2,shift={(-1.1cm,0.2cm)}] {{\footnotesize$to 18$}};

\path	
(a0) edge[loop left]  node[left,ActionStyle,xshift=0cm] {$\adrill,3$}  (a0)
 (a0) edge[]  node[ActionStyle] {$\amine,1$}  (a2)
(a2) edge[]  node[ActionStyle] {$\aload,2$}  (a4)
(a2) edge[]  node[ActionStyle,pos=.2] {$\aload,3$}  (c1)
 (a4) edge[bend right]  node[ActionStyle,xshift=-.2cm] {$\adepo,2$}  (a8)
(a8) edge[]  node[sloped,below,ActionStyle] {$\aunload,2$}  (a10)
(a8) edge[]  node[ActionStyle] {$\aunload,2$}  (a12)
(a12) edge[]  node[ActionStyle] {$\arepair,2$}  (a19)
(a19) edge[]  node[near end,ActionStyle] {$\amine,2$}  (a23)
(a19) edge[loop below]  node[left,ActionStyle] {$\adrill,3$}  (a19)
 (a23) edge[]  node[swap,pos=.8,ActionStyle] {$\aload,2$}  (a4)
(a10) edge[]  node[sloped,above,ActionStyle] {$\arepair,1$}  (a0)
(a23) edge[bend left]  node[sloped,pos=.8,above,ActionStyle] {$\aload,3$}  (a5)
(a5) edge[]  node[above,ActionStyle] {$\adepo,2$}  (a9)
(a9) edge[]  node[swap,near start,ActionStyle] {$\aunload,2$}  (a13)
(a9) edge[]  node[sloped,above,ActionStyle] {$\aunload,2$}  (a14)
(a14) edge[]  node[above,ActionStyle] {$\arepair,2$}  (a18)
(a18) edge[]  node[above,pos=.4,ActionStyle] {$\amine,2$}  (a21)
(a18) edge[loop above]  node[ActionStyle] {$\adrill,3$}  (a18)
(a21) edge[]  node[sloped,above,ActionStyle] {$\aload,2$}  (a24)
(a24) edge[]  node[ActionStyle,xshift=0.2cm] {$\adepo,2$}  (a9)
(a16) edge[swap]  node[very near start,ActionStyle] {$\amine,1$}  (a20)
(a16) edge[swap]  node[sloped,above,ActionStyle] {$\amine,1$}  (a21) 
(a16) edge[loop above]  node[left,ActionStyle,shift={(-.1cm,-.2cm)}] {$\adrill,3$}  (a16)
(a20) edge[]  node[sloped,below,ActionStyle] {$\aload,2$}  (a25)
(a25) edge[]  node[sloped,above,ActionStyle] {$\adepo,1$}  (a7)
(a7) edge[]  node[sloped,above,ActionStyle] {$\aunload,1$}  (a11)
(a0) edge[]  node[swap,ActionStyle,xshift=.1cm] {$\amine,1$}  (a3)
(a3) edge[]  node[sloped,below,ActionStyle] {$\aload,3$}  (a6)
(a6) edge[]  node[sloped,below,ActionStyle] {$\adepo,1$}  (a7)
(a11) edge[bend left=35]  node[swap,ActionStyle,xshift=-.2cm] {$\arepair,3$}  (a0)
(a13) edge[]  node[sloped,below,ActionStyle] {$\arepair,1$}  (a16)

;

%\node (name)[above of=a2,shift={(0cm,0cm)}] {\labelfig{$\CGDES$}};

\end{scope}

\end{tikzpicture}}
\caption{$\CGDES$ for the example in  Figure~\ref{fig:ambientIntel-system}}
\label{ex:CG}
 \vspace{.15in}
\end{figure}

A controller generator $\CG$ for a composition problem is able to \emph{generate} controllers $\Comp$ such that $\Comp(h,\sigma) \in \CG(last(h),\sigma)$, where $h$ is a system history and $\sigma$ is a target action request (cf.~Section~\ref{sec:preliminaries_bc}). 
Similarly, we say here that  a DES controller generator $\CGDES$ \emph{generates} controllers $\Comp$ such that $\Comp(h,\sigma) \in \omega(last(h),\sigma)$.

\begin{example}
Figure~\ref{ex:CG} %% ref not working
represents the complete DES CG for the plant in Figure~\ref{fig:plant}.
For instance, if the user requests action $\aload$ from 
state $1$ (namely, $\tup{t_1,a_3,b_0,c_0}$), the controller generator
allows both behaviors $\B_2$ and $\B_3$ to be scheduled, i.e.,
$\omega(\tup{t_1,a_3,b_0,c_0},\amine)=\set{2,3}$. Therefore, any
controller $\Comp$ generated by $\CGDES$ will be such that
$\Comp(h,\aload)\in\set{2,3}$, where $h$ is any system history such
that $last(h)=\tup{t_1,a_3,b_0,c_0}$.

Note how sequences of transition in the plant (an action request, a
delegation and a behaviour evolution) are compressed in $\CGDES$ into
a single transition. For instance, the sequence 
$g_0\goto{\adrill}g_1\goto{3}g_4\goto{c_0}g_0$ is combined into the 
$\CGDES$'s transition from state $0$ which is labelled $\tup{\adrill,3}$, such that $\omega(0,\adrill)=\set{3}$.  
\end{example} 

The following result demonstrates the correctness of our DES-based approach to compute controller generators.

\begin{theorem}\label{theo:extract_cg_correct}
A controller $\Comp$ is a composition of $\Targ$ in $\S$ iff it is generated by DES CG $\CGDES$. 
\begin{proof}
$(\Rightarrow)$ Assume that there exists a composition $\Comp$ that can not be generated by $\CGDES$. It means that for some target trace $\tau$ and induced history $h\in\H_{\Comp,\tau}$ it is $\Comp(h,\sigma)=j\not\in \CGDES(\last(h),\sigma)$ for some $\sigma$. By construction of $\CGDES$, this means that there is no $p = (\sigma\concat j \concat b'_j) \in  (\Act_t \concat
   \Services \concat \Nondet)$  such that $\rho(\last(h),p)$ is defined
   in $\R$. By definition of $\R$, it contradicts Theorem~\ref{theo:completness}.
$(\Leftarrow)$ Assume that there exists a controller $\Comp$ generated by $\CGDES$ which is not a composition. Similarly, this contradicts
Theorem \ref{theo:soundness}. 
\end{proof} 
\end{theorem}

Note how, from Corollary~\ref{corolllary:1} and the definition of $\CGDES$, we get that $\tup{t,b_1,\ldots,b_n}\in Q$ iff $t \preceqnd
\tup{b_1,\ldots,b_n}$, which matches the
definition of controller generator in~\cite{DeGiacomoPatriziSardina:AIJ13}.%\todo{look}

Also, the DES-based approach is optimal w.r.t. computational complexity.

\begin{theorem}\label{theo:extract_cg_complexity}
Computing the DES controller generator $\CGDES$ can be done in exponential time in the number of available behaviors, and polynomial in their size.
\end{theorem}

The size of the plant $\Plant$ is indeed exponential in the number of behaviors, and the procedure to synthesize the supervisor (that is, to extract $\R$) is polynomial in the size of the plant and the generator for the specification~\cite{WonhamRamadge:SIAMJCO87,GohariWonham:TSMC00}.
It follows then that computing the DES controller generator can be done in exponential time in the number of behaviors, which is the best we can hope for~\cite{DeGiacomoS:IJCAI07}.

\subsection{Implementation in \TCT}\label{sec:TCT}
We close this section by noting that there are, in fact, several tools available for the automated synthesis of supervisors for a discrete event system. The reduction above allows us to use those tools off-the-shelf. 
In  particular, we have used \TCT~\cite{ZhangWonham:SSCDES01}, in which both the plant and the specification are formalized as generators, to extract the generator $\R$ encoding a supervisor for controlling the language $\TSpecc$.  

\newcommand{\Size}{\mathname{Size}}
\renewcommand{\Init}{\mathname{Init}}
\newcommand{\Mark}{\mathname{Mark}}
\newcommand{\Voc}{\mathname{Voc}}
\newcommand{\Tran}{\mathname{Tran}}

In \TCT, generators and recognizers are represented as standard DES in the form of a 5-tuple $\tup{\Size, \Init, \Mark, \Voc, \Tran}$: 
$\Size$ is the number of states (the standard state set is $\set{0,\ldots,Size-1}$); $\Init$ is the initial state (always taken to be $0$); $\Mark$ lists the marker states; $\Voc$ are the vocal states (not needed here); and $\Tran$ are the transitions.  
A transition is a triple $\tup{s,e,s'}$ representing a transition from
$s$ to $s'$ with label $e\in E$, where $E$ is the set of possible
event labels, encoded as integers.  
To distinguish between controllable and uncontrollable events, the
tool assumes that all controllable events are represented with even
integers, uncontrollable ones with odd integers.

For example, the generator $\G$ and the specification $\Spec_2$ from Example \ref{ex:generator2} are encoded in \texttt{ADS} format as follows:

%\begin{multicols}{3}
{\footnotesize

%#Generator G removed
\begin{verbatim}
G
State size: 4
Marker states: 0
Vocal states:

Transitions:
0 1 1
1 3 1
1 0 2
2 9 3
2 5 1
1 7 0
\end{verbatim}

%\columnbreak

%#Specif. K removed
\begin{verbatim}
K
State size: 1
Marker states: 0
Vocal states:

Transitions:
0 1 0
0 3 0
0 0 0
0 5 0
0 7 0

\end{verbatim}

%\columnbreak

\noindent
Events: 
%\noindent
$\texttt{0}$ = $\gBreak$, $\texttt{1}$ = $\gOn$, $\texttt{3}$ =
$\gOperate$,  $\texttt{5}$ = $\gRepair$, $\texttt{7}$ = $\gOff$, $\texttt{9}$ = $\gDismantle$\\
%\newline

}

The encoding is self explanatory and amounts to declaring the plant generator \texttt{G} and the generator for the desired specification \texttt{K}.
A  transition line \texttt{0 1 1} in \texttt{G} encodes the transition $0\goto{\gOn}1$ in Figure \ref{ex:generator2}.
Observe how the generator for \texttt{K} has only one state, and any
event but $\gDismantle$ is represented by a loop on that state. Indeed, recall that this generator captures the specification $\Spec_2$, which excludes  $\gDismantle$ (event \texttt{9}). %

\noindent
The following steps are required for using the \TCT tool to compute the supervisor:

\begin{enumerate}\itemsep=1pt
\item Create the plant \texttt{G} and the specification \texttt{K} in \texttt{ADS} format for $\Plant$ and $\Specc$ (above);

\item Use the \texttt{FD} command to convert \texttt{DES} files in \texttt{ADS} format;
\item Use \texttt{supcon(G,K)} command to compute \texttt{R}, namely the generator representing the supremal controllable sublanguage of \texttt{K}, i.e., $\supremal{\Specc}$;
\item Use \texttt{supreduce(R)} to minimize the generator's size (this step is optional);
\item Compute \emph{control patterns} via \texttt{condat(G,R)}, where a control pattern is the set of (controllable) events that must be disabled in each state of the plant.
\end{enumerate}

\noindent
In this example, the generator $\texttt{R}$ is as the plant in Figure~\ref{ex:generator} but without state $3$. Accordingly, the \TCT command \texttt{condat(G,R)}  outputs \texttt{2:9}, which (only) disables $\gDismantle$ (event \texttt{9}) in state $2$ of $\G$.

% \input{5-BCpreferences}
%%%%%%%%%%%%%%%%%%%%%%%%%%%%%%%%%%%%%%%%%%%%%%%%%%%%%%%%%%%%%%%
\section{Supremal Realizable Target Fragment}\label{sec:DEScomposition-approx}
%%%%%%%%%%%%%%%%%%%%%%%%%%%%%%%%%%%%%%%%%%%%%%%%%%%%%%%%%%%%%%%

\newcommand{\targtrans}{\theta}
\newcommand{\action}{\ensuremath{\textsf{\small act}}} 

Suppose we are given a target behavior $\Targ$ and an available system $\S$ such that no exact composition for $\Targ$ in $\S$ is possible---the target \emph{cannot} be completely realized in the system. 
A mere ``no solution'' answer is unsatisfactory in most cases. 
The need to look beyond exact compositions was first recognized by \citea{StroederPagnucco:IJCAI09}, were they argue that one should look for ``approximations'' in
problem instances not admitting exact compositions.
Then,~\citea{YadavSardina:Jelia2012} provided the first attempt to define and study properties of such approximations.
Subsequently, these optimal approximations were refined and named \emph{supremal
  realizable target fragments} (SRTF) in~\cite{YadavFelliDeGiacomoSardina:IJCAI13}.
In this section, we show how to adapt the composition plant $\Plant$ to look for SRTFs (rather than exact composition) for the special case of deterministic available systems (i.e., one where all available behaviors are deterministic).

Roughly speaking, an SRTF is a ``fragment'' of the target behavior which accommodates an exact composition and is closest to the (original) target module.
It turns out that there is an exact solution for the original target iff there exists an SRTF that is simulation equivalent to it (a property that can be checked in polynomial time). More surprising is the fact that SRTFs are unique (up to simulation equivalence). 
Concretely,~\citea{YadavSardina:Jelia2012} first proposed to allow for \emph{nondeterministic} target behaviors but model user's requests as target transitions (instead of just actions). By doing that, full controllability of the target module is maintained while allowing approximating the original target as much as possible.
The definition of SRTFs, then, relies on the formal notion of simulation~\cite{Miln71}, already discussed in Section \ref{sec:preliminaries_bc}.
A target behavior $\tilde{\T} =
\tup{\tilde{T},\tilde{\BAct_t},\tilde{t_0},\tilde{\delta_t}}$ is a
\defterm{realizable target fragment} (RTF) of original target
specification $\T = \tup{T,\BAct_t,t_0,\delta_t}$ in available system
$\S$ iff 

\begin{itemize}
\item $\tilde{\T}$ is simulated by $\T$ (i.e., $\tilde{\T}   \preceq \T$); and 
\item $\tilde{\T}$ has an exact composition in $\S$.
\end{itemize}

Then, an RTF $\tilde{\T}$ is \defterm{supremal} (SRTF) iff there is no other RTF $\tilde{\T}'$ such that $\tilde{\T} \prec \tilde{\T}'$ (i.e., $\tilde{\T} \preceq \tilde{\T}'$ but $\tilde{\T} \not\preceq \tilde{\T}'$).
Intuitively, a supremal RTF is the closest alternative to the original target that can be completely realized. 
An alternate way of looking at SRTFs is to view them as the (infinite) union of all RTFs~\cite{YadavFelliDeGiacomoSardina:IJCAI13}.

\smallskip
The question we are interested in is as follows: \emph{is it possible to adapt the DES-based composition framework developed above to obtain SRTFs rather than exact compositions?} 

\smallskip
We answer this question positively for the case when available behaviors in $\S$ are deterministic.
The key idea to synthesizing SRTFs by controlling a DES plant is the fact that we are no longer committed to realize \emph{all} target traces: we only need to realize \emph{as many as possible}. 
As a consequence, one can see target actions no more as nondeterministic requests over which we have no control, but instead as actions one may decide to fulfill or not, possibly depending on context. 
Technically this means that events corresponding to user's requests are now assumed \emph{controllable}---the supervisor can enable certain requests and disable others.\footnote{Note that unlike what is often done in a standard DES applications, we do not want to control something in the real world that was not previously controllable, something that usually requires more capabilities (e.g., new actuators). The ``real world" remains the same (i.e., the available behaviors), and we now control what the user will be allowed to potentially request (i.e., the final target specification, the SRTF).}

So, we start by assuming that system $\S=\tup{\B_1,\ldots,\B_n}$ is deterministic and that, following~\cite{YadavSardina:Jelia2012}, target modules $\Targ$ may be, in general, nondeterministic: there may be two  transitions $\tup{t,a,t'}$, $\tup{t,a,t''} \in \delta_t$ such that $t'\neq t''$. To maintain controllability, though, user's requests amount to \emph{target transitions} of the form $\targtrans = \tup{t,a,t'}\in\delta_t$. Still, the task is to implement the action $a$ in the chosen transition $\theta$ via behavior delegation.

So, let us define an alternative DES plant $\PlantApprox$ suitable for synthesising supervisors encoding SRTFs.
Note that, since system $\S$ is deterministic, the whole process for one target request involves now only \emph{two} $\gamma$-transitions, both via controllable events, namely,  $(\targtrans\cdot j)\in (\delta_t \cdot \Services)$.
Note also that the original target is still assumed to be deterministic (the alternative supremal target may be nondeterministic).

\begin{definition}
Let $\S$ and $\T$ be a (deterministic) system and a target module, resp., as in Section~\ref{sec:DEScomposition}.
The \defterm{maximal composition plant} is defined as $\PlantApprox = \tup{\Act,G,g_0,\gamma,G_m}$, where:
\begin{itemize}
\item $\Act = \Act_c \cup \Act_u$ is the set of events of the plant,   where $\Act_u =\emptyset$ and $\Act_c=\Services \cup \delta_t$;

\item $G = T \times B_1 \times \ldots \times B_n \times (\delta_t \cup \set{e})$
  is the set of plant states;

\item $g_0 = \tup{t_0, b_{01},\ldots,b_{0n},e}$ is the initial state of $\PlantApprox$;

\item $\gamma : G \times \Act \ra G$ is the plant's transition function where  
$\gamma(\tup{t,b_1,\ldots,b_n,\req},\sigma)$ is equal to:
\begin{itemize}
\item $\tup{t,b_1,\ldots,b_n,\sigma}$ if $\req = e \text{ and } \sigma \in \delta_t$;
\item $\tup{t',b_1,\ldots,b_{\sigma}',\ldots,b_n,e}$ if $\req =
  \tup{t,a,t'} \in\delta_t, \sigma \in \Services$ and $b_{\sigma}' = \delta_{\sigma}(b_{\sigma},a)$.
\end{itemize}

\item $G_m = T \times B_1 \times \ldots \times B_n \times
  \set{e}$.
\end{itemize}
\end{definition}

\noindent
Notably, both target transition requests and behavior delegations are now controllable: the supervisor is allowed to forbid (i.e., disable) target requests.

As before, we just take $\Speccaprox = \Mkd(\PlantApprox)$ as the specification language to control (in the maximal composition plant).
To build a SRTF for target $\T$ in system $\S$, we first compute the language $\TSpeccaprox$, and then build its corresponding generator $\hat{\R}= \tup{ \Act, Y, y_0 , \rho, Y_m }$. 
Finally, we extract $\hat{\R}$ the alternative, possibly nondeterministic, target behavior $\MaxTarg_{\tup{\S,\T}} = \tup{T^*,\BAct_t,y_0,\delta_t^*}$, where:
\begin{itemize}
  \item $T^* = \set{ y \mid y \in Y, p\in(\delta_t \cdot \Services)^*, y =\rho(y_0,p) }$; and

    \item $\delta^*_t \subseteq T^* \times \BAct_t \times T^* $ is such that $y' \in \delta^*_t(y,a)$
    \emph{iff} 
    $y' = \rho(y,\targtrans \cdot j)$, where $j \in \Services$, $\targtrans \in \delta_t$, and $\targtrans = \tup{t,a,t'}$.
\end{itemize}

Next, we relate system histories to
words of plant $\PlantApprox$.
Differently from before, however, histories do not hold enough
information for this purpose, because the target is now nondeterministic. 
We then make use of the following definitions, to relate a target
trace $\tau$ \emph{and} an induced
system history $h\in\H_{\Comp,\tau}$ (for some $\Comp$) to a word in $\Gtd(\PlantApprox)$.
Given a target trace $\tau = t_0\goto{a_1}\cdots\goto{a_\ell} t_\ell$,
the word $\word(\tau,h)$ corresponding to an induced system history $h =
\vec{b}_0\goto{a_1,j_1}\cdots\goto{a_\ell,j_\ell}\vec{b}_\ell$
is 
\[	
\word(\tau,h) \defeq ( \tup{t_0,a_1,t_1} \cdot j_1) \cdot \; \ldots \; \cdot ( \tup{t_{\ell-1},a_\ell,t_\ell} \cdot j_\ell).
\]

We now present the key results for our technique.

\begin{lemma}
\label{lemma2}
$\Comp$ is a composition for an RTF $\tilde{\Targ}$ of $\Targ$ in $\S$
iff $\set{\word(\tau,h) ~|~ \tau\in\tilde{\Targ} \land h\in
\H_{\Comp,\tau}}\subseteq \TSpeccaprox$.
\end{lemma}

\begin{proof}
($\Rightarrow$) Assume by contradiction
that there exists a composition $\Comp$ for $\tilde{\Targ}$ in $\S$ such that for some target trace
$\tau = t_0\goto{a_1}t_1\goto{a_2}\cdots\goto{a_\ell} t_\ell$ of
$\tilde{\Targ}$ and
induced history $h\in\H_{\Comp,\tau}$, we have $P(h,\targtrans)=j$ but
$\word(\tau,h) \cdot \targtrans \cdot  j \not\in \TSpeccaprox$, where
$\targtrans = \tup{t_\ell,a,t_{\ell+1}}$ is the new target transition
been requested.
This implies that $\word(\tau,h) \cdot \targtrans \cdot j $ is not
allowed by  $\sup$, from the initial state
$g_0$.  
Similarly to Lemma~\ref{lemma:final}, either 
\begin{itemize}\itemsep=0pt
\item[$(1)$] $\word(\tau,h) \concat \targtrans \not\in
\Gtd(\PlantApprox)$ or 
\item[$(2)$] $\word(\tau,h) \concat \targtrans \concat j \not\in
\Gtd(\PlantApprox)$ or
\item[$(3)$] for all words $w\in\Mkd(\PlantApprox)$ with $w>\word(\tau,h) \concat \targtrans \concat
j$ 
 we have $w\not\in\Closure{\TSpeccaprox}$. 
\end{itemize}
Without loss of generality, assume $\targtrans=\tup{t_\ell,a,t_{\ell+1}}$. If $(1)$ is true, then from the definition of $\PlantApprox$
(transition function $\gamma$, first item) the transition $\targtrans$ is not in
$\delta_t$, thus $\tilde{\Targ}$ is not an RTF of $\Targ$, as
$\tilde{\Targ}\not\preceq\Targ$. Similarly,
case $(2)$ violates the assumption that $\Comp$ is a composition of
$\tilde{\Targ}$, as behaviour $j$ is not able to perform the action
$a$ from its current local state. Indeed, any legal delegation of
action $a$ to any behaviour is considered in $\PlantApprox$
(transition function $\gamma$, second item). Case $(3)$ can be
excluded by considering cases $(1)(2)$ by induction on $\ell$.

($\Leftarrow$) Assume by contradiction this is not the case. Then there exists a word
$\word(\tau,h)\in\TSpeccaprox$ for some trace $\tau$ of an RTF
$\tilde{\Targ}$ of $\Targ$ in $\S$ and history $h\in\H_{\Comp,\tau}$
such that $\Comp$ is not a composition of $\tilde{\Targ}$ in $\S$.
It means that for some $\theta=\tup{\stateproj_t(\last(h)),a,t'}$ in
$\tilde{\Targ}$ we have $\Comp(h,\theta)=j$
but behavior $j$ can not perform action $a$, namely there is no $b'_j$
such that $b'_j\in\delta_j(\stateproj_j(last(h)),a)$.
Hence, according to $\PlantApprox$ (and its transition function
$\gamma$), $\word(h,\tau)\concat\theta\concat
j\not\in\Gtd(\PlantApprox)$, and we get a contradiction. 
\end{proof}

\begin{theorem}[Soundness]
Let $\T$ be a target and $\S$ a deterministic system. Then, $\MaxTarg_{\tup{\S,\T}}$ is a SRTF of $\T$ in $\S$.
\end{theorem}

\begin{proof} 
By construction $\MaxTarg_{\tup{\S,\T}}$ is an RTF of $\T$ in $\S$.
The proof is left to the reader, as it is evident by the fact that $\Gtd(\hat{\R})=\Closure{\TSpeccaprox}\subseteq\Gtd(\PlantApprox)$, and transitions in $\PlantApprox$ are only defined wrt $\Targ$'s evolution (function $\delta_t$).

It remains to show that $\MaxTarg_{\tup{\S,\T}}$ is indeed the maximal
one. 
Suppose $\T^{\uparrow}$ is the SRTF of $\T$ in $\S$ and $\MaxTarg_{\tup{\S,\T}} \not \succeq \T^{\uparrow}$.
Therefore, there exists a trace $\tau = t_0 \goto{a_1}\cdots
\goto{a_\ell}  t_\ell\goto{a_{\ell+1}}t_{\ell+1}$ of $\T^{\uparrow}$ 
that is not in $\MaxTarg_{\tup{\S,\T}}$,
and the simulation ``breaks'' at $t_\ell$. Formally, for any trace $t^*_0\goto{a_1}\cdots\goto{a_\ell}t^*_\ell$
of $\MaxTarg_{\tup{\S,\T}}$ with $t^*_0=y_0$ and the same action sequence $a_1\cdots
a_\ell$ as $\tau$, there is no transition
$t^*_\ell\goto{a_{\ell+1}}t^*_{\ell+1}$: $\MaxTarg_{\tup{\S,\T}}$ is
unable to perform action $a_{\ell+1}$.
However, by Lemma~\ref{lemma2}, the word $\word(\tau,h)$ is in
$\TSpeccaprox$, 
and the trace $t^*_0
\goto{a_1}\cdots\goto{a_{\ell+1}}t^*_{\ell+1}$ has to be a trace of
$\MaxTarg_{\tup{\S,\T}}$. So we get a contradiction. Thus
$\T^{\uparrow}$ does not exist and $\MaxTarg_{\tup{\S,\T}} $ is an
SRTF of $\T$ in $\S$.
\end{proof}

Observe, the controller generator for $\MaxTarg_{\tup{\S,\T}}$ can be computed \emph{while} building the SRTF itself, by tracking delegations in $\hat{\R}$, similar to DES-based composition.
In fact, we have:

\begin{theorem}[Completeness]
Let $\sup$ be a nonblocking supervisor such that $\Mkd(\sup/\PlantApprox)=\TSpeccaprox$.
Then, every composition $\Comp$ for $\MaxTarg_{\tup{\S,\T}}$ in system $\S$ can be induced by $\sup$.
\end{theorem}
\begin{proof}
Assume by contradiction that there exists a composition $\Comp'$ which can not be induced by $\sup$, i.e., it is such that $\Comp'(h,\theta) \not\in \sup(\word(\tau,h)\cdot\theta)$ for some target transition $\theta=\tup{t^*_{\ell-1},a_\ell,t^*_\ell}$ and target trace $\tau=t^*_0\goto{a_1}\cdots\goto{a_\ell} t^*_\ell$ in $\MaxTarg_{\tup{\S,\T}}$ and some induced system history
$h\in\H_{\Comp',\tau}$. 
Assume $\Comp'(h,\theta)=j$. It follows that  $\word(\tau,h)\cdot
\theta \cdot j \not\in \TSpeccaprox$, which contradicts Lemma~\ref{lemma2}.
\end{proof}

In words, every supervisor that can control language $\TSpeccaprox$ in
the maximal composition plant $\PlantApprox$ encodes all exact
compositions of the SRTF $\MaxTarg_{\tup{\S,\T}}$ built
above.

\begin{corollary}%\todo{look at this}
\label{corolllary:1} 
Given a plant $\PlantApprox$ for $\S$ and $\T$ as above, if $\TSpeccaprox\neq\emptyset$ then for any word $w\in\TSpeccaprox$, we have that
$\stateproj_t(g)\preceqnd\tup{\stateproj_1(g),\ldots,\stateproj_n(g)}$
where $g=\gamma(g_0,w)$.
\end{corollary} 

The proof proceeds as in the basic case.

\begin{example}
Figure~\ref{fig:ambientIntel-system-sup} depicts the SRTF for a deterministic variant of the problem instance of Figure~\ref{fig:ambientIntel-system}, in which the excavator can not load the truck twice without being repaired.
Observe that $\MaxTarg_{\tup{\S,\T}}$ and $\Targ$ are simulation equivalent, hence we know that an exact composition exists for the original specification $\Spec$. However, since actions $\aload$ and $\arepair$ are now nondeterministic  in $\MaxTarg_{\tup{\S,\T}}$ (from $t_1$ and $t_7$, respectively), one could, in principle, chose to delegate a nondeterministic action to different behaviors based on the transition requested (cf. Figure~\ref{ex:CG}). 
Hence, the $\MaxTarg_{\tup{\S,\T}}$  is now more informative than
$\Targ$, in the sense that it allows for more solutions if the user
commits to subsequent choices (here, whether to dig next or not). The
so-obtained representation of the SRTF is not minimal, but can be
easily minimised (e.g., by using the {\small\texttt{supreduce}}
command in TCT --- see Section \ref{sec:TCT}.)
\end{example}
 
%%%%%%%%%%%%%%%% EXAMPLE %%%%%%%%%%
\begin{figure}[!t]
\begin{center}
\resizebox{\columnwidth}{!}{\begin{tikzpicture}[->,thin,node distance=1.8cm,double distance=2pt]
\tikzstyle{every state}=[circle,fill=none,draw=black,text=black,inner sep=1pt,minimum size=4mm,font=\small]

\tikzset{ActionStyle/.style = {font=\footnotesize}}

%%% TRUCK
\begin{scope}[shift={(0cm,-.1cm)},node distance=2cm,auto]
\node[initial above,state]	(a0)                    {$a_0$};
\node[state]    			(a1) [right of=a0] 		{$a_1$};
\node[state]                (a2) [below left of=a1,yshift=.05in]      {$a_2$};

\path	
(a0) edge[loop left]  node[sloped,below,xshift=-.7cm,ActionStyle] {$\arepair$}    (a0)
 (a0) edge[bend left]  node[swap,ActionStyle,yshift=-.1cm] {$\amine$}    (a1)
(a1) edge[bend left=20]  node[pos=.4,ActionStyle,xshift=-.1cm] {$\adepo$}    (a2)
(a2) edge[bend left]  node[swap,pos=.4,ActionStyle] {$\aunload$} (a0)
;

\node (name)[above of=a0,shift={(0cm,-1.2cm)}]	
	{{\footnotesize \labelfig{Truck $\B_{1}$}}};
\end{scope}

%%EXCAVATOR
\begin{scope}[shift={(7.5cm,.7cm)},node distance=2cm,auto]
\node[initial right,state]	(c0)                    {$c_0$};
\node[state]    			(c1) [below of=c0,shift={(0cm,.6cm)}] 		{$c_1$};

\path	
(c0) edge[loop left]  node[below,pos=.3,ActionStyle] {$\adrill$}    (c0)
(c0) edge[bend left]  node[sloped,above,ActionStyle] {$\aload$}    (c1)
(c1) edge[bend left]  node[pos=0.4,ActionStyle] {$\arepair$}    (c0)
;

\node (name)[below of=c1,shift={(-.3cm,1.4cm)}]	
	{{\footnotesize \labelfig{Excavator $\B_{3}$}}};
\end{scope}

%%% LOADER
\begin{scope}[shift={(5.5cm,0cm)},node distance=2cm,auto]
\node[initial right,state]	(b0)                    {$b_0$};
\node[state]    			(b1) [left of=b0,xshift=-0.3cm] 		{$b_1$};
\node[state]                (b2) [below right of=b1,yshift=.15in,xshift=-.2cm]      {$b_2$};

\path
(b0) edge[loop above]  node[ActionStyle] {$\aload$}    (b0)	
(b0) edge[bend right]  node[swap,ActionStyle] {$\adepo$}    (b1)
(b1) edge[]  node[swap,ActionStyle] {$\aunload$}    (b0)
(b1) edge[bend right]  node[swap,pos=.9,ActionStyle] {$\arepair$}    (b2)
(b2) edge[sloped]  node[below,ActionStyle] {$\amine$}    (b0)
;

\node (name)[above of=b1,shift={(-.6cm,-1.2cm)}]	
	{{\footnotesize \labelfig{Loader $\B_{2}$}}};
\end{scope}

%%% TARGET
\begin{scope}[shift={(0cm,-2.5cm)},node distance=1.7cm,auto]

\node[initial below,state]	(t0)   {$t_0$};
\node[state]	(t1)   [right of=t0,xshift=.5cm]   {$t_1$};
\node[state]	(t2)   [right of=t1]   {$t_2$};
\node[state]	(t3)   [right of=t2,xshift=.5cm]   {$t_3$};
\node[state]	(t4)   [right of=t3,xshift=.1cm]   {$t_4$};
\node[state]	(t5)   [below right of=t0,shift={(.8cm,.2cm)}]   {$t_5$};
\node[state]	(t6)   [right of=t5,shift={(.3cm,-.2cm)}]   {$t_6$};
\node[state]	(t7)   [right of=t6,xshift=.1cm]   {$t_7$};
\node[state]	(t8)   [right of=t7,xshift=-.1cm]   {$t_8$};
%\node[state]	(t9)   [right of=t8,xshift=.1cm]   {$t_9$};

\path
(t0) edge[loop left]  node[below,pos=.2,ActionStyle] {$\adrill$}    (t0)	
(t0) edge[]  node[swap,pos=.5,ActionStyle] {$\amine$}    (t1)	
(t1) edge[]  node[ActionStyle] {$\aload$}    (t2)	
(t2) edge[]  node[ActionStyle] {$\adepo$}    (t3)	
(t3) edge[]  node[swap,ActionStyle] {$\aunload$}    (t4)	
(t4) edge[bend right=20]  node[ActionStyle] {$\arepair$}    (t0)	
(t1) edge[]  node[pos=.3,ActionStyle] {$\aload$}    (t5)	
(t5) edge[]  node[sloped,below,ActionStyle,swap] {$\adepo$}    (t6)	
(t6) edge[]  node[ActionStyle] {$\aunload$}    (t7)	
(t7) edge[]  node[ActionStyle,swap] {$\arepair$}    (t8)	
(t8) edge[]  node[sloped,above,pos=.3,ActionStyle] {$\amine$}    (t1)
(t7) edge[bend left=34]  node[sloped,below,pos=.85,ActionStyle] {$\arepair$}    (t0)	
;

\node (name)[above of=t0,shift={(0cm,-1cm)}]	
	{{\footnotesize \labelfig{$\MaxTarg_{\tup{\S,\T}}$}}};
\end{scope}

\end{tikzpicture}}
\end{center}
\vspace*{-.3cm}
\caption{SRTF for a deterministic variant ($\T$ as in Figure~\ref{fig:ambientIntel-system}).}
\label{fig:ambientIntel-system-sup}
\end{figure}
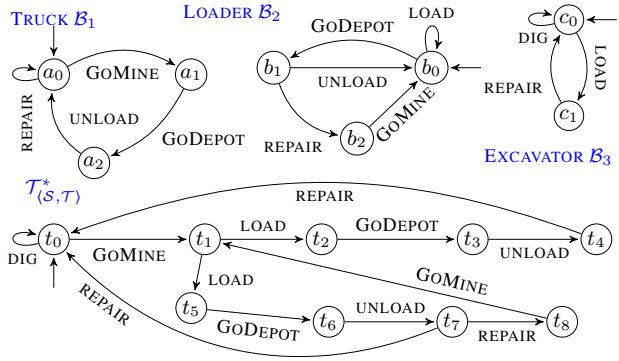
%%%%%%%%%%%%%%%%%%%%%%%%%%%%%%%%%%

%%%%%%%%%%%%%%%%%%%%%%%%%%%%%%%%%%%%%%%%%%%%%%%%%%%%%%%%%%%%%%%
\section{Conclusions} \label{sec:conclusions}
%%%%%%%%%%%%%%%%%%%%%%%%%%%%%%%%%%%%%%%%%%%%%%%%%%%%%%%%%%%%%%%
\newcommand{\DES}{DES\xspace}

From an Computer Science perspective, planning, SCT, and behavior composition are all synthesis problems: build a plan, supervisor, or controller, respectively.
Observe that, at the core, these problems are concerned with
qualitative temporal decision making in dynamic domains and exhibit
strong resemblances in how their problem components are modeled (e.g.,
using transition system-like models) and the solution techniques used 
%to solve the problem 
(e.g., model checking, search, etc.).
In fact, exploration of the relationship between these three synthesis tasks has already gained attention~\cite{Balbianietal:SERVICES09,Bertoli2010:PlanningComp,Barbeau1995:DESPlanning}.

The behavior composition problem can be considered as a planning problem for a maintenance goal, namely, to \emph{always} satisfy the target's request.
There, a plan (i.e., a controller) prescribes behavior delegations rather than domain actions~\cite{RamirezYadavSardina:ICAPS13}.
In particular, various forms of composition problems have been considered under various planning frameworks, including planning as model checking~\cite{Pistore2004:PlanningComp}, planning in asynchronous domains~\cite{Bertoli2010:PlanningComp}, and nondeterministic planning~\cite{RamirezYadavSardina:ICAPS13}.   
 
With respect to SCT, planning techniques have been used for both synthesis of supervisors~\cite{Barbeau1995:DESPlanning,Barbeau1997:DESBacktracking} as well as for diagnosis problems~\cite{Grastien2007:DESdiagnosis}. 
However, the link between composition and SCT still remains unexplored. To our knowledge, the only available literature deals with showing the decidability of mediator synthesis for web-service composition by reduction to DES~\cite{Balbianietal:SERVICES09}. 
Such work considers a composition setting involving web services able to exchange messages, and the task is to synthesise a \emph{mediator} able to communicate with them to realize the target specification, instead of an orchestrator (i.e., a controller) that schedules \nolinebreak them.

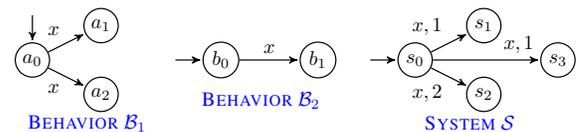
\begin{figure}[!b] 
\centering
\resizebox{.9\columnwidth}{!}{ 
\begin{tikzpicture}
 \tikzset{ActionStyle/.style = {font=\normalsize}}
\tikzstyle{every state}=[circle,fill=none,draw=black,text=black,inner sep=2pt,minimum size=6mm,font=\normalsize]

%%%% BEHAVIOR B1
\begin{scope}[shift={(0cm,0cm)},node distance=1.7cm,auto]
\node[initial above,state]	(b0)                    		{$a_{0}$};
\node[state]    			(b1) [above right of=b0,yshift=-.6cm] 		{$a_{1}$};
\node[state]                (b2) [below right of=b0,yshift=.6cm]      	{$a_{2}$};

\path 	(b0) edge[]             	node[ActionStyle] 					{$x$}    (b1)
        (b0) edge[]             	node[ActionStyle,swap]				{$x$}    (b2)
        ;

\node (name)[below of=b2,shift={(-.25cm,1.2cm)}]	
	{{\normalsize \labelfig{Behavior $\B_1$}}};
\end{scope}

%%%% BEHAVIOR B2
\begin{scope}[shift={(3.3cm,0cm)},node distance=1.7cm,auto]
\node[initial left,state]	(b0)                    		{$b_{0}$};
\node[state]    			(b1) [right of=b0] 				{$b_{1}$};

\path 	(b0) edge[]             	node[ActionStyle] 					{$x$}    (b1)
        ;

\node (name)[below of=b1,shift={(-1cm,1cm)}]	
	{{\normalsize \labelfig{Behavior $\B_2$}}};
\end{scope}

%%%% SYSTEM B2
\begin{scope}[shift={(6.7cm,0cm)},node distance=1.7cm,auto]
\node[initial left,state]	(b0)                    		{$s_{0}$};
\node[state]    			(b1) [above right of=b0,yshift=-.6cm] 		{$s_{1}$};
\node[state]                (b2) [below right of=b0,yshift=.6cm]      	{$s_{2}$};
\node[state]                (b3) [right of=b0,xshift=.8cm]		      	{$s_{3}$};

\path 	(b0) edge[]             	node[ActionStyle] 					{$x,1$}    (b1)
        (b0) edge[]             	node[ActionStyle,swap]				{$x,2$}    (b2)
        (b0) edge[]             	node[ActionStyle,pos=.8]				{$x,1$}    (b3)
        ;

\node (name)[below of=b2,shift={(-.25cm,1.2cm)}]	
	{{\normalsize \labelfig{System $\S$}}};
\end{scope}
\end{tikzpicture}
}
\caption{Comparing DES and Behavior Composition.}
\label{fig:des-example}
\end{figure}

From the outset, it may seem that behavior composition and SCT are tackling the same problem, though maybe from different perspectives: SCT from an Engineering perspective and composition from a Computer Science one. 
Nonetheless, the inherent control problem in SCT and behavior composition are different in nature. 
In the latter, one seeks to control the available behaviors, whereas in the former one can prevent (some of the) actions. 
Consider the simple example shown in Figure~\ref{fig:des-example} with a nondeterministic behavior $\B_1$ and a deterministic behavior $\B_2$.
See that both behaviors share the action $x$; hence, in the enacted system $\S$, $x$ will be nondeterministic for $\B_1$ but not for $\B_2$ (as shown by the indexes used in $\S$).
The input in SCT is the whole plant, and it does not have a notion
analogous to available behaviors. Therefore,  component-based
nondeterminism cannot be captured (directly) in a plant, and one has
to make delegation events (i.e., indexes) explicit in the plant, as we showed in this paper.
% %
Another important mismatch has to do with the semantics of nondeterminism: 
the nondeterminism of controllable actions in a plant is \emph{angelic}, in the sense that the supervisor can control its evolution. On the other hand, nondeterminism of available behaviors is \emph{devilish}, as it cannot be controlled. 
This is one of the reasons why, as far as we know, \DES frameworks do not have a notion similar to \textsc{nd}-simulation~\cite{DeGiacomoPatriziSardina:AIJ13}.
Indeed, uncertainty is modelled here via (deterministic) uncontrollable events~\cite{WonhamRamadge:SIAMJCO87}, whereas nondeterminism~\cite{DeGiacomoPatriziSardina:AIJ13} is used to model uncertainty (and partial controllability) in behavior composition.
Lastly, the term ``composition'' itself differs considerably: in SCT it refers to synchronous product between sub-systems, instead of an asynchronous realization in behavior composition literature.

Notwithstanding all the above differences, this paper shows that a link can indeed be drawn. In particular, we have demonstrated that solving an AI behavior composition problem can be seen as finding a supervisor for a certain plant. In doing so, one can expect to leverage on the solid foundations and extensive work in SCT, as well as on the tools available in those communities.
We have shown, for instance, that the DES-based encoding can accommodate (meta-level) constraint on the composition in a straightforward manner. In addition, we detailed how to slightly adapt the encoding to look for ``the best possible'' target realization when a perfect one does not exist, though only for the case of deterministic systems.
For practical applicability, experimental work should follow the work presented here  to check whether existing tools in SCT provide any advantages over existing composition techniques via game solvers~\cite{DeGiacomoPatrizi:WSFM09,DeGiacomoFelli:AAMAS10} or automated planners~\cite{RamirezYadavSardina:ICAPS13}.

Once the formal relationship between the two different synthesis tasks has been established, many possibilities for future work open up.
In fact, we would like to import notions and techniques common in SCT into the composition setting, such as hierarchical and tolerance supervision/composition~\cite{Cassandras:BOOK06-DES}. 
An interesting aspect to look at is how to use the marked language of specification $\Spec$ in order to encode \emph{constraints}. % than the ones we can represent now. 
For example, one may want to impose that certain complex (high-level) tasks or processes built from domain action executions may not be started if their termination cannot be guaranteed. 
So, some goods in a factory production chain should not  be
cleaned unless it is guaranteed that they will be packaged and disposed afterwards. So far, we have only used the marked language (of the plant) to force complete termination of each action-request and delegation step.
On the other direction, probably the most interesting aspect to explore is the use of automated planning systems and game solvers to solve DES problems.

\footnotesize
\bibliography{bibstrings,paolo,behcomposition,desbiblio}
\bibliographystyle{plainnat}

% that's all folks
\end{document}